\documentclass[10pt,journal,compsoc]{IEEEtran}

\ifCLASSOPTIONcompsoc
  \usepackage[nocompress]{cite}
\else
  \usepackage{cite}
\fi
\usepackage{wrapfig}
\usepackage{amsmath}
\usepackage{amsfonts}
\usepackage{bm}
\usepackage{subcaption}
\usepackage{mathtools}
\usepackage{amsthm}
\usepackage{multirow}
\usepackage{xcolor}
\usepackage{booktabs}
\usepackage{nicefrac}
\usepackage{url}
\usepackage{hyperref}
\usepackage{amssymb}
\usepackage{algorithm}
\usepackage{algorithmic}
\usepackage{enumitem}
\usepackage{cancel}
\usepackage{url}

\ifCLASSINFOpdf
\else
\fi
\usepackage{array}

\usepackage{stfloats}

\newcommand{\loss}{\mathcal{L}}

\newcommand{\udfsection}[1]{\noindent\textbf{#1}\, }

\newtheorem{assumption}{Assumption}
\newtheorem{lemma}{Lemma}
\newtheorem{proposition}{Proposition}
\newtheorem{theorem}{Theorem}

\newcommand{\method}{LD}
\newcommand{\extractor}{f}
\newcommand{\normalize}{\textsc{NORMALIZE}}
\newcommand{\softmax}{\textsc{SOFTMAX}}
\newcommand{\sampler}{\textsc{SAMPLER}}
\newcommand{\GNN}{\textsc{GNN}}
\newcommand{\MLP}{\textsc{MLP}}
\newcommand{\GAT}{\textsc{GAT}}
\newcommand{\mat}[1]{\mathbf{#1}}

\newif\ifproof\prooftrue

\newif\ifupdate\updatefalse

\newif\ifupdateok\updateoktrue
\newcommand{\modifyok}[2]{\ifupdate{#1}\else{\color{black}#2}\fi}

\newif\ifswitch\switchtrue

\begin{document}
%
\title{Label Deconvolution for Node Representation Learning on Large-scale Attributed Graphs against Learning Bias}

%
%
%
%

\author{
Zhihao~Shi,~Jie~Wang,~\IEEEmembership{Senior Member,~IEEE,}~Fanghua~Lu,~Hanzhu~Chen,~Defu~Lian,~\IEEEmembership{Member,~IEEE,}\\Zheng~Wang,~\IEEEmembership{Member,~IEEE,}~Jieping~Ye,~\IEEEmembership{Fellow,~IEEE,}~and~Feng Wu,~\IEEEmembership{Fellow,~IEEE}

\IEEEcompsocitemizethanks{
\IEEEcompsocthanksitem Jie Wang is with MoE Key Laboratory of Brain-inspired Intelligent Perception and Cognition, University of Science and Technology of China, Hefei 230027, China. (Corresponding author, e-mail: jiewangx@ustc.edu.cn).
\IEEEcompsocthanksitem Zhihao Shi, Fanghua Lu, Hanzhu Chen, Defu Lian, and Feng Wu are with MoE Key Laboratory of Brain-inspired Intelligent Perception and Cognition, University of Science and Technology of China, Hefei 230027, China. E-mail: \{zhihaoshi,fanghualu,hanzhuchen\}@mail.ustc.edu.cn, \{liandefu,fengwu\}@ustc.edu.cn.
\IEEEcompsocthanksitem Zheng Wang and Jieping Ye are with the Alibaba Group, Hangzhou 310030, China. E-mail: \{wz388779, yejieping.ye\}@alibaba-inc.com.
}

\thanks{Manuscript received April 19, 2005; revised August 26, 2015.}}

%
%

\markboth{Journal of \LaTeX\ Class Files,~Vol.~14, No.~8, August~2015}%
{Shell \MakeLowercase{\textit{et al.}}: Bare Advanced Demo of IEEEtran.cls for IEEE Computer Society Journals}
%



\IEEEtitleabstractindextext{%
\begin{abstract}


Node representation learning on attributed graphs---whose nodes are associated with rich attributes (e.g., texts and protein sequences)---plays a crucial role in many important downstream tasks. To encode the attributes and graph structures simultaneously, recent studies integrate pre-trained models with graph neural networks (GNNs), where pre-trained models serve as node encoders (NEs) to encode the attributes. As jointly training large NEs and GNNs on large-scale graphs suffers from severe scalability issues, many methods propose to train NEs and GNNs separately. Consequently, they do not take feature convolutions in GNNs into consideration in the training phase of NEs, leading to a significant learning bias relative to the joint training. To address this challenge, we propose an efficient label regularization technique, namely \textbf{L}abel \textbf{D}econvolution (\method{}), to alleviate the learning bias by a novel and highly scalable approximation to the inverse mapping of GNNs. The inverse mapping leads to an objective function that is equivalent to that by the joint training, while it can effectively incorporate GNNs in the training phase of NEs against the learning bias. More importantly, we show that LD converges to the optimal objective function values by the joint training under mild assumptions. Experiments demonstrate LD significantly outperforms state-of-the-art methods on Open Graph Benchmark datasets.

\end{abstract}

\begin{IEEEkeywords}
Graph Neural Networks, Label Deconvolution, Pre-trained Models, Attributed Graphs, Node Feature Extraction
\end{IEEEkeywords}}

\maketitle

\IEEEdisplaynontitleabstractindextext

%
\IEEEpeerreviewmaketitle

\ifCLASSOPTIONcompsoc
\IEEEraisesectionheading{\section{Introduction}\label{sec1}}
\else
\section{Introduction}\label{sec1}
\fi

%
%
%
%

\IEEEPARstart{G}{raphs} are widely used in many important fields, such as citation networks \cite{mag, aminer, lpa}, co-purchase networks \cite{amazon, cluster_gcn}, and protein–protein association networks \cite{string, gor}.
In many real-world applications, nodes in graphs are associated with rich and useful attributes.
For example, nodes in citation networks, co-purchase networks, and protein-protein association networks are often associated with titles/abstracts, textual descriptions for products, and protein sequences.
Many powerful pre-trained models could capture a wealth of information about node properties from the node attributes.
In this paper, we focus on node representation learning on the attributed graphs, which plays an important role in many downstream tasks, such as node classification and link prediction.

To encode the node attributes and graph structures simultaneously, a commonly seen architecture is to integrate powerful pre-trained models with graph neural networks (GNNs)  \cite{giant, glem, textgnn, text_level_gnn},  as shown in Fig. \ref{fig:pipeline}.
First, the pre-trained models---which serve as node encoders (NEs)---encode the node attributes into low-dimensional node features.
Then, graph neural networks take both the node features and graph structures as input to iteratively update node representations \cite{mpnn, grl}.
In summary, NEs encode the local attributes of individual nodes, and GNNs learn the global structural relationship among nodes.

However, on large-scale graphs, jointly training both NEs and GNNs sacrifices either the model capacity of NEs or the size of graph structures for GNNs, resulting in severe performance degradation.
Specifically, in terms of NEs, an idea is to use small NEs or handcrafted feature extraction \cite{ogb} based on node attributes to reduce the expensive costs of the joint training.
As the larger NEs empirically learn more prior knowledge from a large corpus, large NEs become promising in many real-world applications against small NEs or handcrafted feature extraction.
In terms of GNNs, recent works propose various graph sampling techniques to alleviate the scalability issues of GNNs.
The integration of large NEs with graph sampling reduces the size of graph structures \cite{textgnn, text_level_gnn}, severely sacrificing graph topological information encoded by GNNs (see Table \ref{tab:subgraph_size}).

To preserve the high capacity of NEs and address the severe scalability issue of GNNs, many existing methods propose to train NEs and GNNs separately \cite{glem, giant}.
Specifically, they first ignore feature convolutions in GNNs to train NEs (the training phase of NEs) and then directly train GNNs based on the fixed NEs (the training phase of GNNs).
During the training phase of NEs, ignoring feature convolutions in NEs avoids the notorious neighbor explosion issue \cite{graphsage} of GNNs in large-scale graphs.
During the training phase of GNNs, fixed NEs avoid the expensive update of vast parameters of large NEs, allowing existing scalable GNNs to large-scale graphs.
Therefore, the separate training framework allows existing mini-batch techniques to train NEs and GNNs respectively without sacrificing the model capacity of NEs or the graph structures for GNNs.


Nevertheless, they suffer from a significant learning bias relative to the joint training due to the neglect of feature convolutions in GNNs.
By noticing that feature convolutions encode graph structures to predict node labels, we summarize the learning bias in terms of the node labels and the graph structures.
Specifically, during the training phase of NEs, the supervision signals for NEs used in many separate training frameworks consist of either the node labels or the graph structures,
while the supervision signals of the joint training consist of both the node labels and the graph structural information in backward passes (see Fig. \ref{fig:comparison}).
In terms of node labels, some related works propose a scalable self-supervision task termed neighborhood prediction to incorporate graph structural information into NEs \cite{giant}.
As the self-supervision task neglects the node labels, the extracted node features may contain much task-irrelevant information and hence hurt node representations of GNNs.
In terms of the graph structures, some other works directly use the node labels to train NEs \cite{glem}, independent of graph structures.
However, the labels may be noisy for NEs, as they depend on both node attributes and graph structures.
For example, nodes with similar attributes and different structures may be associated with different node labels.

\begin{figure*}[t]
    \centering
    \includegraphics[width=\textwidth]{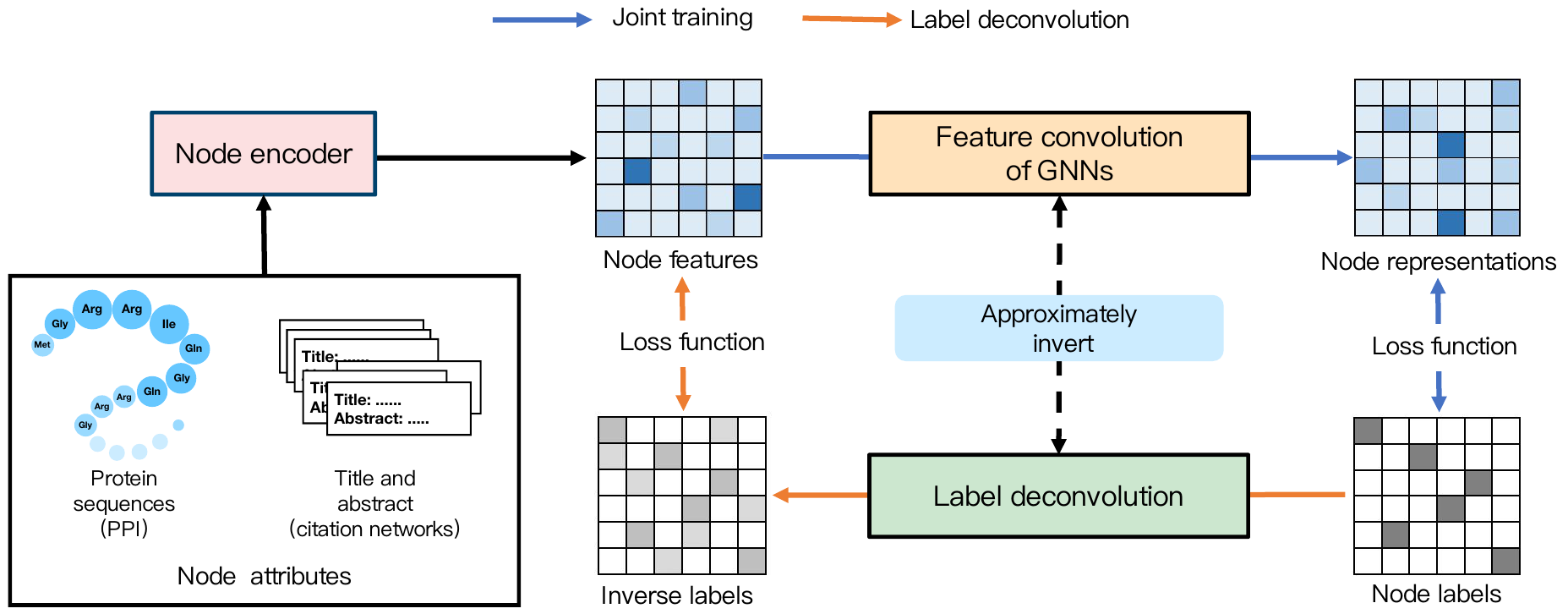}
    \caption{Label deconvolution introduces a scalable objective function to approximate the original objective function based on the concept of the inverse mapping of GNNs, which preserves the optimal solution and thus alleviates the learning bias relative to the joint training.
    }
    \label{fig:pipeline}
\end{figure*}

To address this challenge, we propose an efficient and effective label regularization technique, namely Label Deconvolution (LD), to alleviate the learning bias by integrating GNNs with NEs in the training phase of NEs.
As shown by Fig. \ref{fig:pipeline}, given the true node labels for GNNs used in the joint training, LD aims to find equivalent labels for NEs (denoted by inverse labels) based on an approximate inverse mapping of GNNs.
The inverse mapping guarantees that if we fit the output of NEs (denoted by node features) to the inverse labels, then the corresponding output of GNNs (denoted by node representations) is fitted to the true node labels.
More importantly, we show that LD converges to the optimal objective function values by the joint training for spectral-based GNNs under some mild assumptions (see Assumption \ref{ass:attributes_labels} about our observation that the node labels depend on both node attributes and graph structures).
To efficiently approximate the inverse mapping of GNNs, LD first decouples the memory and time-consuming feature convolution from GNNs inspired by \cite{sgc, gamlp, ssgc}.
By noticing that the node labels are fixed during training, LD pre-processes the memory and time-consuming inverse mapping of the feature convolution only once based on the fixed node labels, unlike the joint training, which performs the feature convolution at each training step to encode the learnable features from NEs.
Extensive experiments demonstrate \method{} significantly outperforms state-of-the-art methods by a significant margin on Open Graph Benchmark datasets.

\section{Preliminaries}

We introduce the problem setting in Section \ref{sec:problem_setting}. Then, we introduce the joint training methods based on graph sampling in Section \ref{sec:scalable_gnn}. Next, we introduce the spectral graph neural networks in Section \ref{sec:spe_gnns}. Finally, we introduce the separate training framework to train GNN and NEs efficiently in Section \ref{sec:separate_training}.

\subsection{Node Representation Learning on Large-scale Attributed Graphs} \label{sec:problem_setting}

We focus on node representation learning on a graph $\mathcal{G}=(\mathcal{V}, \mathcal{E})$ with rich and useful node attributes $\{ \mat{x}_i \}_{i=1}^{|\mathcal{V}|}$, where $\mathcal{V}$ is the set of all nodes and $\mathcal{E}$ is the set of all edges.
As the true node attributes $\{ \mat{x}_i \}_{i=1}^{|\mathcal{V}|}$ usually are high-dimensional texts, images, or protein sequences, a commonly-seen solution is to first extract $d_f$-dimensional node features $\{ \mat{f}_i \}_{i=1}^{|\mathcal{V}|}$ from them as follows
\begin{align}\label{eqn:ne}
    \mat{f}_i^{(\beta)} = \extractor(\mat{x}_i; \beta) \in \mathbb{R}^{d_f},\,i=1,2,\dots,|\mathcal{V}|,
\end{align}
where $\beta$ represents the parameters of the node encoders (NEs). 
Large pre-trained models (e.g. ESM2 \cite{esm2} for protein sequences and Bert \cite{bert, deberta} for texts) serve as NEs $f$ due to their powerful ability to extract features \cite{giant, glem}.

To further encode graph structures, graph neural networks take both nodes features $\mat{F}^{(\beta)} \in \mathbb{R}^{ |\mathcal{V}| \times d_f }$ and adjacent matrix $\mat{A} \in \mathbb{R}^{ |\mathcal{V}| \times |\mathcal{V}|}$ as inputs as follows
 \begin{align}\label{eqn:gnn}
     \mat{H} = \GNN(\mat{F}^{(\beta)};\mat{A},\theta) \in \mathbb{R}^{  |\mathcal{V}| \times d },
 \end{align}
 where $\mat{F}_{i,:}^{(\beta)} = \mat{f}_i^{(\beta)}$ denotes the $i$-th row of $\mat{F}^{(\beta)}$ and $\theta$ represents the parameters of the graph neural network. $\mat{A}_{ij}=1$ if $(i,j) \in \mathcal{E}$ and otherwise $\mat{A}_{ij}=0$.
Then, GNNs output node representation $\mat{H}$.

For simplicity, we define the following notations. Given a set of nodes $\mathcal{B}$, let $\mat{M}_{\mathcal{B},:} = (\mat{M}_{i,:})_{i \in \mathcal{B}}$ denote the matrix consisting of $\mat{M}_{i,:}$ with all $i \in  \mathcal{B}$, where $\mat{M}_{i,:}$ is the $i$-th row of $\mat{M}$. Given a vector function $f: \mathbb{R}^{d_1} \rightarrow  \mathbb{R}^{d_2}$, we overload $\mat{M}' = f(\mat{M})$ to denote a matrix function with ${\mat{M}}_{i,:}' = f({\mat{M}}_{i,:})$.

\subsection{Scalable Graph Neural Networks}\label{sec:scalable_gnn}

Many scalable graph neural networks are categorized into two ideas in terms of data sampling and model architectures, respectively.

\udfsection{Graph Sampling for GNNs.} Equation \eqref{eqn:gnn} takes all node features as inputs and hence it is not scalable.
To compute node representations in a mini-batch $\mathcal{B}$ of nodes, a commonly-seen solution is to sample a subgraph constructed by $\mathcal{B}$ as follows,
\begin{align}\label{eqn:gnn_minibatch}
     \mat{H}_{\mathcal{B},:} &= \GNN( \mat{F}_{G(\mathcal{B}),:};\mat{A}_{G(\mathcal{B})},\theta) \in \mathbb{R}^{ |\mathcal{B}| \times d };\\
      \mat{F}_{G(\mathcal{B}),:}, \mat{A}_{G(\mathcal{B})} &= \sampler(\mat{F}, \mat{A})  ,
\end{align}
where $|\mathcal{B}| << |G(\mathcal{B})| << |\mathcal{V}| $.
Notably, $|G(\mathcal{B})|$ used in existing graph sampling methods is significantly larger than the size of the mini-batch used in pre-trained NEs.
If we further decrease the size of $|\mathcal{B}|$ or $|G(\mathcal{B})|$ in existing graph sampling methods to align their batch sizes, their performance will significantly drop, as shown in Table \ref{tab:subgraph_size}.
In experiments, the max batch size of pre-trained NEs is at most 12 (see Table \ref{tab:runtime}), which is significantly smaller than $|G(\mathcal{B})|$.
Therefore, the joint training of NEs and GNNs by graph sampling is usually unaffordable.

\begin{table}[t]
\caption{Prediction performance with different sizes of sampled subgraphs used in our experiments. We directly infer node features by the pre-trained models without fine-tuning. For GAS \cite{gas}, we increase the number of subgraph partitions to decrease $|G(\mathcal{B})|$ on the ogbn-arxiv dataset. For SAGE \cite{graphsage}, we first decrease the number of target nodes $|\mathcal{B}|$ and then decrease the sampled neighbors to decrease $|G(\mathcal{B})|$ on the ogbn-protein dataset.}
\centering
\label{tab:subgraph_size}
\scalebox{0.95}{
\begin{tabular}{cccccc}
\toprule
\multirow{3}[1]{*}{GCN+GAS \cite{gas}} &\textbf{$|\mathcal{B}|$} & 84,260   & 25,574  & 2,566 & 264 \\
&\textbf{$|G(\mathcal{B})|$} & 96,536   &  48,871  & 4,691 &  3,958 \\ 
&\textbf{ACC}   & 74.28\% & 74.14\% & 74.07\% & 67.46\% \\ 
\midrule
\multirow{3}[1]{*}{GAT+SAGE \cite{graphsage}} &\textbf{$|\mathcal{B}|$} & 8,662   & 866  & 87 & 87 \\
&\textbf{$|G(\mathcal{B})|$} &  131,381  & 131,119  & 130,948  & 80,529 \\
&\textbf{AUC}  & 89.06\% & 88.74\% & 84.40\% & 81.27\% \\ 
\bottomrule
\end{tabular}}
\end{table}

\udfsection{Decoupling Feature Convolution from GNNs as a Pre-processing Step.}
To avoid the memory and time-consuming feature convolution of GNNs, other scalable GNNs (e.g. GAMLP \cite{gamlp} and SAGN \cite{sagn}) first decouple the feature convolution from GNNs.
Then, they pre-process the memory and time-consuming feature convolution only once based on fixed node features.
However, as the node features are learnable by NEs, the idea is still unaffordable for the joint training of NEs and GNNs.

\subsection{Spectral-based GNNs} \label{sec:spe_gnns}

Many GNNs \cite{gcn, gcnii, linear_gnn} are inspired by spectral filters and thus many studies analyze the properties of GNNs based on their spectral-based approximation.
The spectral-based GNNs \cite{linear_gnn} are
\begin{align}\label{eqn:spectral_gnn}
    \mat{H} = \GNN(\mat{F}^{(\beta)};\mat{A},\theta)  = \phi (\hat{\mat{A}}; \theta^{(\phi)}) \psi( \mat{F}^{(\beta)}; \theta^{(\psi)} ),
\end{align}
where $\phi (\hat{\mat{A}}; \theta^{(\phi)})= \sum_{i=0}^{N} \theta_i^{(\phi)} \hat{\mat{A}}^i$ is a polynomial spectral filter to perform linear feature convolutions, $\hat{\mat{A}}$ is the normalized adjacent matrix,
and $\psi$ is a non-linear multi-layer perceptron.
The weights $\theta^{(\phi)}$ are either learnable or fixed.

The spectral-based GNN is a reasonable approximation of various GNNs.
We elaborate on a connection between a wide range of GNNs (e.g., GCN \cite{gcn}, REVGAT \cite{deq_gcn}, GAMLP \cite{gamlp}, SAGN \cite{sagn}, and GAT \cite{gat}  in the experiments) and the spectral-based GNN in Appendix A.2.
Theoretically, we show that the spectral-based GNN is powerful enough to produce arbitrary node predictions under the assumptions of the no-multiple-eigenvalue and no-missing-frequency conditions as shown in Appendix A.3.
Moreover, the assumptions of the no-multiple-eigenvalue and no-missing-frequency conditions hold on many real-world graphs \cite{linear_gnn}.
Empirically, Fig. 8 in Appendix A.4 shows that given a GNN, the corresponding spectral-based approximation can achieve a similar prediction performance.

\subsection{Separate Training Frameworks}\label{sec:separate_training}

To avoid the expensive costs of jointly training both NEs and GNNs (the vast parameters of large NEs and the neighbor explosion issue of GNNs), recent studies propose a separate training framework to train GNNs and NEs respectively \cite{glem, giant}.

Given node labels $\mat{Y}$, the optimization problem is $\min_{\theta, \beta} \loss (\GNN(\mat{F}^{(\beta)}, \mat{A};\theta) , \mat{Y})$. To avoid severe scalability issue of feature convolutions, separate training frameworks alternately optimize $\theta$ and $\beta$ by
\begin{align} \label{eqn:obj_gnn}
    \min_{\theta} \loss (\GNN(\mat{F}^{(\beta)};\mat{A},\theta) , \mat{Y}),
\end{align}
and
\begin{align} \label{eqn:obj_ne}
    \min_{\beta} \loss (\GNN(\mat{F}^{(\beta)}; \mat{A},\theta) , \mat{Y}),
\end{align}
where $\loss$ is the loss function for the true objective function.

\udfsection{Training Phase of GNNs (Optimize $\theta$).} In Equation \eqref{eqn:obj_gnn}, as the parameters of NEs $\beta$ are fixed, the scalable GNNs in Section \ref{sec:scalable_gnn} are applicable without the reduction of graph sizes based on the fixed node features $\mat{F}^{(\beta)}$.

\udfsection{Training Phase of NEs (Optimize $\beta$).} On the right side of Equation \eqref{eqn:obj_ne}, the neighbor explosion issue remains unchanged as shown in Equation \eqref{eqn:gnn_minibatch}. Thus, many separate training frameworks ignore feature convolutions in GNNs to design a new loss function $\loss'$ to $\loss$
\begin{align}\label{eqn:obj_exist}
     \min_{\beta} \loss' (\mat{F}^{(\beta)} , \mat{Y}, \mat{A}) \approx  \min_{\beta} \loss (\GNN(\mat{F}^{(\beta)}; \mat{A},\theta) , \mat{Y}),
\end{align}
such as a self-supervision loss $\loss_{ssl} (\mat{F}^{(\beta)} ,\mat{A})$ \cite{giant} or a supervision loss $\loss (\psi(\mat{F}^{(\beta)}) ,\mat{Y})$ \cite{glem} with a scalable linear layer $\psi:\mathbb{R}^{d_f} \rightarrow \mathbb{R}^{d}$.

Our formulation is different from GLEM \cite{glem}, as LD and GLEM are based on different motivations.
Specifically, LD aims to recover $\GNN$ in Equation \eqref{eqn:obj_ne}, while GLEM aims to improve the quality of the pseudo labels $\mat{Y}_{\mathcal{V}_{test},:}$ on the test nodes $\mathcal{V}_{test}$  for semi-supervised learning.
Thus, we omit the improvement of $\mat{Y}$ and assume that the node labels $\mat{Y}$ used in LD and GLEM are the same, as LD adopts the approach proposed by GLEM to improve the quality of the pseudo labels for semi-supervised learning. 

\section{Label Deconvolution}


In this section, we elaborate on the proposed Label Deconvolution (LD), a simple and efficient label regularization technique. First, we introduce the motivation of LD, based on the concept of the inverse mapping of graph neural networks (GNNs). Next, we present a fast approximation algorithm to implement this inverse mapping. Finally, we highlight the advantages of LD over existing separate training methods.


\subsection{Motivation for Label Deconvolution}\label{sec:motivation}

As shown in Section \ref{sec:separate_training}, existing separate training methods propose many approximate loss functions \ref{eqn:obj_exist} during the training phase of NEs. However, the approximation does not ensure optimality and thus suffers from a significant learning bias.
We give a counterexample to show that existing loss functions do not ensure optimality in Section \ref{sec:example}.



To alleviate the learning bias, we derive an equivalent loss function based on the based on the concept of the inverse mapping of GNNs, which ensures the optimality of the proposed loss function.
Specifically, if $\GNN$ is invertible  and $L$-Lipschitz continuous, i.e., for all $\mathbf{F}^{1}, \mathbf{F}^{2}$
\begin{align*}
    \| \GNN(\mathbf{F}^{1};\mathbf{A},\theta) - \GNN(\mathbf{F}^{2};\mathbf{A},\theta) \|_F \leq L \| \mathbf{F}^{1} - \mathbf{F}^{2} \|_F, 
\end{align*}
then we have
\begin{align*}
    \| \GNN(\mathbf{F}^{(\beta)};\mathbf{A},\theta) - \mathbf{Y} \|_F \leq L \| \mathbf{F}^{(\beta)} - \GNN^{-1}(\mathbf{Y};\mathbf{A},\theta) \|_F,
\end{align*}
where $\GNN^{-1}$ is the inverse mapping of $\GNN$ and $\mat{Y}$ is the node labels in the whole graph.
Thus, the minimization of  $\| \GNN(\mathbf{F}^{(\beta)};\mathbf{A},\theta) - \mathbf{Y} \|_F$ is bounded by an equivalent minimization of $\| \mathbf{F}^{(\beta)} - \GNN^{-1}(\mathbf{Y};\mathbf{A},\theta) \|_F$.

The feature convolution $\GNN( \mathbf{F}^{(\beta)}; \mat{A}, \theta)$ and the proposed label deconvolution $\GNN^{-1}( \mat{Y}; \mat{A}, \theta)$ obtain equivalent objective functions, while their different inputs can affect the scalability.
During the training phase of NEs, the feature $\mathbf{F}^{(\beta)}$ is learnable with $\beta$ while the label $\mat{Y}$ and the GNN parameters $\theta$ are fixed. 
Thus, we can preprocess $\GNN^{-1}( \mat{Y}; \mat{A}, \theta)$ once and reuse the results during the training phase of NEs, which avoids the memory and time-consuming operation many times.


Inspired by the observation, we replace Frobenius norm $\|\cdot\|_F$ with the objective function \eqref{eqn:obj_ne}, i.e.,
\begin{align} \label{eqn:obj}
    &\min_{\beta} \loss (\mat{F}^{(\beta)} , \hat{\mat{Y}} );\\ \label{eqn:ld}
    &\text{s.t.}\,  \hat{\mat{Y}}  = \GNN^{-1}( \mat{Y}; \mat{A}, \theta) \,\text{(pre-processing)},
\end{align}
where $\GNN^{-1}$ is the inverse mapping of $\GNN$. We call $\hat{\mat{Y}}$ the inverse labels.
Thus, the mini-batch version of the objective function \eqref{eqn:obj} is
\begin{align*}
    \min_{\beta} \loss (\mat{F}^{(\beta)}_{\mathcal{B},:} , \hat{\mat{Y}}_{\mathcal{B},:} ),
\end{align*}
where $\mathcal{B}$ is the mini-batch of nodes.

We further simplify the objective function \eqref{eqn:obj} with the spectral formulation of GNNs \eqref{eqn:spectral_gnn} introduced in Section \ref{sec:spe_gnns}.
Combining $\GNN(\mat{F}^{(\beta)};\mat{A},\theta)  = \phi (\hat{\mat{A}}; \theta^{(\phi)}) \psi( \mat{F}^{(\beta)}; \theta^{(\psi)} )$ with the objective function \eqref{eqn:obj} leads to
\begin{align} \label{eqn:obj_linear}
    &\min_{\beta, \theta } \loss (\psi( \mat{F}^{(\beta)};\theta) , \hat{\mat{Y}} );\\ \nonumber
    &\text{s.t.}\,  \hat{\mat{Y}}  = \phi^{-1} (\hat{\mat{A}}; \theta) \mat{Y}. \,\text{(pre-processing)}
\end{align}
Equation \eqref{eqn:obj_linear} preserves the scalable non-linear transformations of GNNs and pre-process the inverse of linear feature convolutions $\phi^{-1} (\hat{\mat{A}}; \theta)$.

Notably, Equation \eqref{eqn:obj_linear} incorporates a part of the GNN parameters $\theta$ during the training phase of NEs. The incorporation significantly alleviates the learning bias from that by the joint training of NEs and GNNs, while it is still scalable.

\subsection{Fast Approximation of Inverse Mapping}\label{sec:label_deconvolution}

To further avoid the inverse mapping of linear feature convolutions, we propose a trainable label deconvolution to generate the inverse labels $ \mat{Y}^{(\gamma)}$.
Label deconvolution aims to parameterize $ \mat{Y}^{(\gamma)}$ with $\gamma$ such that the expressiveness of $ \mat{Y}^{(\gamma)}$ is similar to $\phi^{-1} (\hat{\mat{A}}; \theta) \mat{Y}$, i.e.,
\begin{align*}
    \{ \mat{Y}^{(\gamma)} : \gamma \} \approx \{ \phi (\hat{\mat{A}}; \theta)^{-1}\mat{Y}: \theta  \}.
\end{align*}
Thus, Equation \eqref{eqn:obj_linear} becomes
\begin{align} \label{eqn:obj_ld}
    &\min_{\beta, \theta, \gamma} \loss (\psi( \mat{F}^{(\beta)};\theta) , \mat{Y}^{(\gamma)} ).
\end{align}
Comparing with Equation \eqref{eqn:obj_linear}, Equation \eqref{eqn:obj_ld} implicitly incorporates parameters $\theta^{(\phi)}$ by our proposed reparameterization approach with $\gamma$.
We provide the theoretical analysis in Section \ref{sec:theory}.

The key idea is inspired by the Cayley-Hamilton theorem. We first introduce two useful lemmas as follows.
\begin{lemma}{\cite{Cayley-Hamilton1, Cayley-Hamilton3}}\label{lemma:inverse}
    Let the characteristic polynomial of a matrix $\mat{M}$ be $f(\lambda) = p_n \lambda^n + p_{n-1} \lambda^{n-1} + \cdots + p_1 \lambda + p_0$.
    If the matrix $\mat{M}$ is invertible, then the inverse of $\mat{M}$ is
    \begin{align*}
        \mat{M}^{-1} = - \frac{1}{p_0}(p_n \mat{M}^{n-1} + p_{n-1} \mat{M}^{n-1} + \cdots + p_2 \mat{M} + p_1 \mat{I}).
    \end{align*}
\end{lemma}
\begin{lemma}{\cite{Cayley-Hamilton2, Cayley-Hamilton3}}\label{lemma:poly}
    The matrix $\mat{M}^N \in \mathbb{R}^{n \times n}$ can be expressed as a matrix polynomial of degree less than $n$, i.e.,
    \begin{align*}
        \mat{M}^N = \sum_{i=0}^{n-1}p_i \mat{M}^i.
    \end{align*}
\end{lemma}
The proposition follows immediately.
\begin{proposition}\label{prop:expressiveness}
    If $\phi(\hat{\mat{A}}; \theta)^{-1}$ is invertible, then $\phi(\hat{\mat{A}}; \theta)^{-1}$ is expressed as a linear combination of the matrix powers of $\hat{\mat{A}}$, i.e., $\phi(\hat{\mat{A}}; \theta)^{-1} = \sum_{i=0}^{|\mathcal{V}|-1} \gamma_i \hat{\mat{A}}^i$.
\end{proposition}

Thus, we parameterize the inverse labels $\mat{Y}^{(\gamma)}$ by
\begin{align}\label{eqn:inv_label}
    \mat{Y}^{(\gamma)} = \sum_{i=0}^N \gamma_i \mat{K}_i = \sum_{i=0}^N \gamma_i \hat{\mat{A}}^i \mat{Y},
\end{align}
where $N$ is a hyper-parameter and the variables $\gamma_i \in \mathbb{R}$ are trainable parameters.

Intuitively, the $i$-hop labels $\mat{K}_i = \hat{\mat{A}}^i \mat{Y}$ are the (weighted) average of the labels in $k$-hop neighbors.
For an $N$-layer GNN, the prediction (representation) of a node not only depends on its feature but also its $N$-hop neighbors' features.
Similarly, the feature of a node not only contributes to its prediction but also its $N$-hop neighbors' predictions.
Therefore, the $i$-hop labels are effective in alleviating the learning bias during the training phase of NEs.


The mini-batch version of Equation \eqref{eqn:inv_label} is
\begin{align}\label{eqn:inv_label_minibatch}
    \hat{\mat{Y}}_{\mathcal{B}}^{(\gamma)} &= \sum_{i=0}^N \gamma_i [\mat{K}_i]_{\mathcal{B},:} ;\\ \label{eqn:preprocess}
    \mat{K}_i &= (\hat{\mat{A}})^i \mat{Y}. \,\text{(pre-processing)}
\end{align}
where $\mathcal{B}$ is the mini-batch of nodes.

\begin{figure*}[t]
    \centering
    \begin{subfigure}{0.24\textwidth}
        \includegraphics[width=\textwidth]{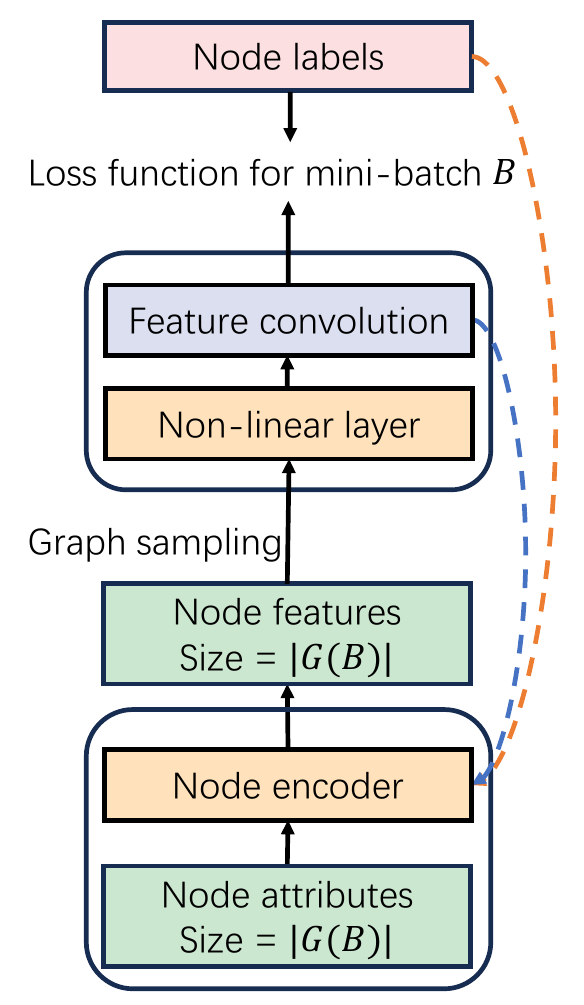}
        \caption{Joint training.}
        \label{fig:end2end}
    \end{subfigure}
    \begin{subfigure}{0.24\textwidth}
        \includegraphics[width=\textwidth]{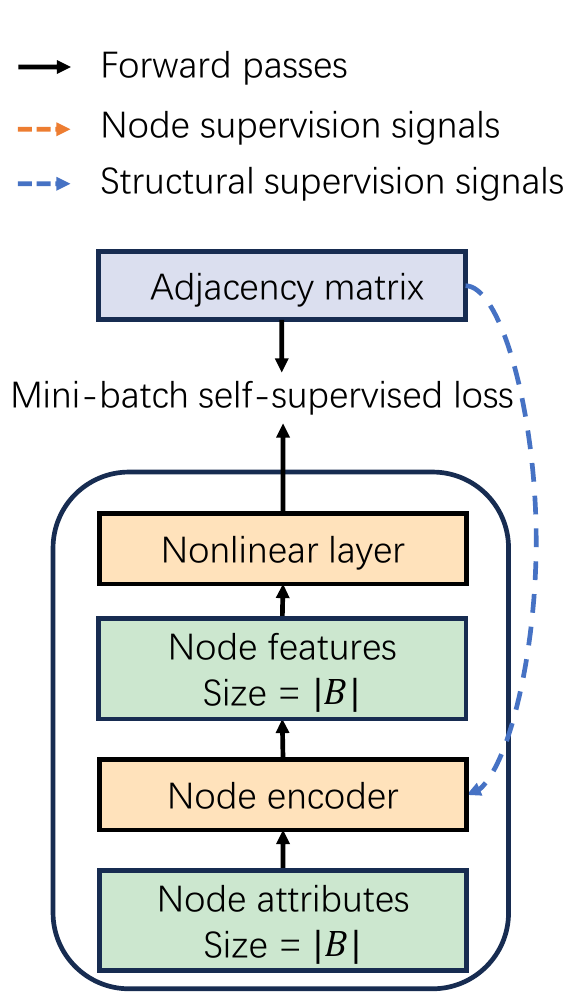}
        \caption{GIANT.}
        \label{fig:giant}
    \end{subfigure}
    \begin{subfigure}{0.24\textwidth}
        \includegraphics[width=\textwidth]{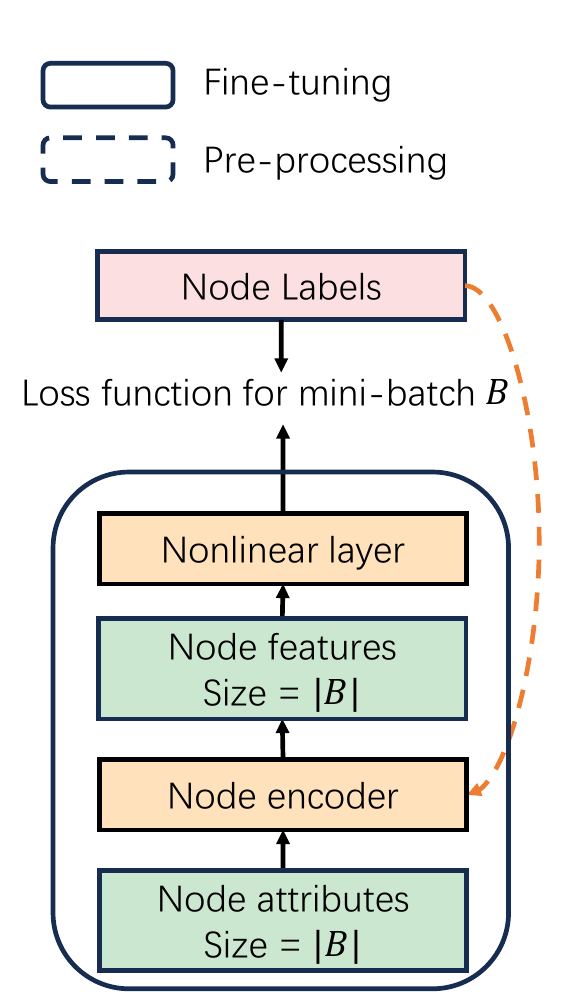}
        \caption{GLEM.}
        \label{fig:glem}
    \end{subfigure}
    \begin{subfigure}{0.24\textwidth}
        \includegraphics[width=\textwidth]{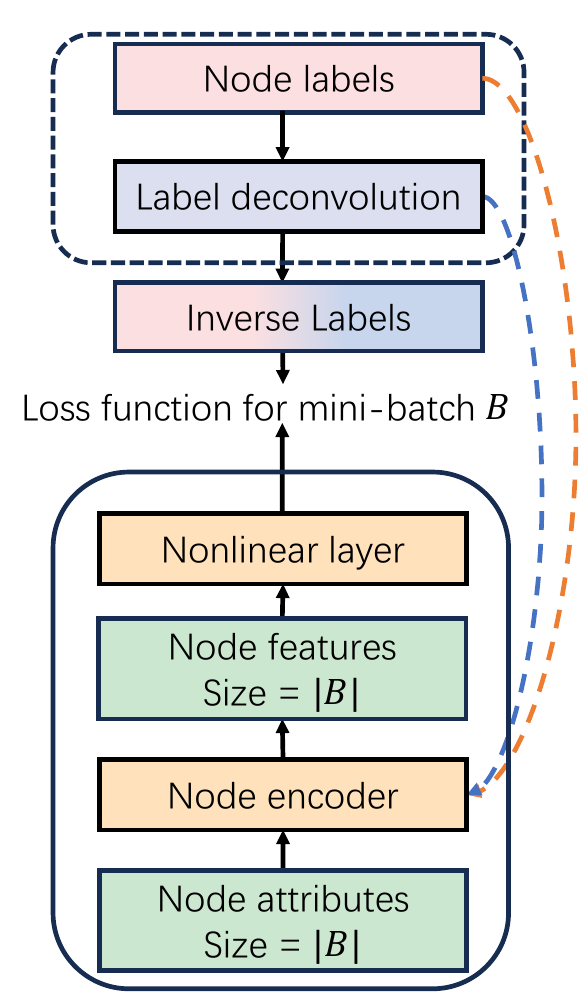}
        \caption{Label deconvolution.}
        \label{fig:ld}
    \end{subfigure}
    \caption{Comparison of different loss functions for NEs.}
    \label{fig:comparison}
\end{figure*}

We summarize our algorithms in Algorithm \ref{alg:ld}. Given node labels be $\mat{Y}$, we first compute and save the $i$-hop labels $\mat{K}_i$ via Equation \eqref{eqn:preprocess} during the pre-processing phase. During the training phase of NEs, we can efficiently compute the inverse labels $\mat{Y}^{(\gamma)}_{\mathcal{B}}$ and use them to train NEs and non-linear transformations $\psi$ of GNNs.
Notably, we simultaneously optimize $\beta$, $\theta$, and $\gamma$ rather than alternately optimize.
Finally, during the training phase of GNNs, we train scalable GNNs via Equation \eqref{eqn:obj_gnn}.

\begin{algorithm}[h]
    \caption{Label Deconvolution}
    \label{alg:ld}
    \begin{algorithmic}[1]
        \STATE {\bfseries Input:} 
        Graph $\mat{A}$, attributes $\{\mat{x}_i\}$, (psuedo) labels $\mat{Y}$.
        \STATE {\bfseries Output:} 
        Node encoder with parameters $\beta$ and GNN with parameters $\theta$.
        \STATE \textit{// Pre-processing Phase:}
        \STATE Compute $\mat{K}_i$ via Equation \eqref{eqn:preprocess}.
        \STATE \textit{// Training Phase of NEs (Optimize $\beta, \theta$):}
        \FOR{{\small$k = 1, \dots, K_1$}}
            \STATE Randomly sample $\mathcal{B}$ from $\mathcal{V}$
            \STATE Compute the node features $\mat{F}^{(\beta)}_{\mathcal{B}}$ via Equation \eqref{eqn:ne}.
            \STATE Compute the inverse labels $\mat{Y}^{(\gamma)}_{\mathcal{B}}$ via Equation \eqref{eqn:inv_label_minibatch}.
            \STATE Optimize $\beta, \theta, \gamma$ via the loss $\loss(\psi(\mat{F}^{(\beta)}_{\mathcal{B}};\theta), \mat{Y}^{(\gamma)}_{\mathcal{B}})$
        \ENDFOR
        \STATE Infer the node features $\mat{F}^{(\beta)}$ across all nodes.
        \STATE \textit{// Training Phase of GNNs (Optimize $\theta$ with fixed $\beta$):}
        \FOR{{\small$k = 1, \dots, K_2$}}
            \STATE Compute the node features $\mat{H}_{\mathcal{B}}$ via Equation \eqref{eqn:gnn_minibatch}
            \STATE Optimize $\theta$ via the loss $\loss(\mat{H}_{\mathcal{B}}, \mat{Y}_{\mathcal{B}})$
        \ENDFOR
    \end{algorithmic}
\end{algorithm}

\subsection{Comparision with Different Loss Functions for NEs}\label{sec:comparison}

As shown in Section \ref{sec:separate_training}, many existing separate training frameworks propose various loss functions to approximate the right side of Equation \eqref{eqn:obj_ne}, leading to a significant learning bias relative to the joint training, which directly optimizes $\min_{\theta, \beta} \loss (\GNN(\mat{F}^{(\beta)}, \mat{A};\theta), \mat{Y})$.
We summarize the learning bias in terms of node labels and graph structures.
Fig. \ref{fig:comparison} shows the loss functions of the joint training, LD, GIANT \cite{giant}, and GLEM \cite{glem}.
LD integrates the graph structures with the node labels to generate the inverse labels, maintaining a similar learning behavior to that of the joint training.
However, GIANT and GLEM neglect either the graph structures or the node labels, resulting in a significant learning bias.

Although LD and the joint training share a similar learning learning behavior, LD is more memory-efficient than the joint training. Specifically, to compute the loss on a mini-batch of nodes $\mathcal{B}$, the NEs of LD encode the attributes in the mini-batch $\mathcal{B}$ with the memory complexity $\mathcal{O}(|\mathcal{B}|)$. However, the NEs of the joint training encode the attributes in a sampled subgraph with size $G(\mathcal{B})$, leading to larger memory complexity $\mathcal{O}(|G(\mathcal{B})|)$ than LD.

Table \ref{tab:complexity_conv} shows the complexity of different training methods during the training phase of NEs and the supervision signals for NEs.
LD and GLEM are the fastest and most memory-efficient algorithms among all methods.
More importantly, LD incorporates graph structures $\mat{A}$ into the supervision signals for NEs, compared with GLEM.

\begin{table}[t]
    \centering
    \caption{
    Computational complexity of different training methods during the training phase of NEs and the supervision signals for NEs.
    $\|\mat{A}\|_0$ is the number of edges and satisfies $|\mathcal{V}| \leq \|\mat{A}\|_0 \leq |\mathcal{V}|^2$.
    }
    \label{tab:complexity_conv}
    \scalebox{1.0}{
    \begin{tabular}{cccc}
    \toprule
         \multirow{2}[1]{*}{\textbf{Methods}}  & \textbf{Time}  &\textbf{Memory} & \textbf{Supervision}  \\
         &  \textbf{Complexity} &  \textbf{Complexity} & \textbf{Signals for NEs} \\
        \midrule
        Joint training  & $\mathcal{O}(|\mathcal{V}|\frac{|G(\mathcal{B})|}{|\mathcal{B}|})$ & $\mathcal{O}(|G(\mathcal{B})|)$ & $\mat{Y}$ \& $\mat{A}$ \\
        GIANT \cite{giant}  & $\mathcal{O}(\|\mat{A}\|_0)$ & $\mathcal{O}(|\mathcal{B}|)$ & $\mat{A}$ \\
        GLEM \cite{glem} & $\mathcal{O}(|\mathcal{V}|)$ & $\mathcal{O}(|\mathcal{B}|)$ &  $\mat{Y}$  \\
        \midrule
        LD & $\mathcal{O}(|\mathcal{V}|)$ & $\mathcal{O}(|\mathcal{B}|)$ & $\mat{Y}$ \& $\mat{A}$ \\
    \bottomrule
    \end{tabular}
    }
\end{table}

\section{Theoretical Results}\label{sec:theory}

In this section, we analyze the learning behaviors of LD. Specifically, we show that LD converges to the optimal objective function values by the joint training under some mild assumption, while GLEM---one of the state-of-the-art separate training frameworks---may converge to the sub-optimal objective function values.
We provide detailed proofs of the theorems in Appendix B.

We formulate the training process of LD as
\begin{align*}
    & \min_{ \theta, \beta_{LD}} \, \loss(\GNN(\mat{F}^{(\beta_{LD})}, \mat{A};\theta), \mat{Y});\\
    & \text{s.t.}\, \beta_{LD} \in \arg\min_{\beta} \min_{\gamma, \theta} \loss(\psi(\mat{F}^{(\beta)}; \theta),  \mat{Y}^{(\gamma)}).
\end{align*}
Similarly, the training process of GLEM is
\begin{align*}
    & \min_{ \theta, \beta_{GLEM}} \, \loss(\GNN(\mat{F}^{(\beta_{GLEM})}, \mat{A};\theta), \mat{Y});\\
    & \text{s.t.}\, \beta_{GLEM} \in \arg\min_{\beta} \min_{\theta}  \loss(\psi(\mat{F}^{(\beta)}; \theta),  \mat{Y}).
\end{align*}
For simplicity, the loss function $\loss$ is a mean squared loss.

\subsection{Motivating Example}\label{sec:example}

Consider a graph $\mathcal{G} = (\mathcal{V}, \mathcal{E}) = (\{1,2,3,4\}, \{(1,2),(2,1),(3,4),(4,3)\})$ with one-hot node attributes $\mat{X} = (\mat{e}_1, \mat{e}_2, \mat{e}_2, \mat{e}_3)^{\top}$. The $i$-th element of  $\mat{e}_i \in \mathbb{R}^3$ is one and the others are zero. The corresponding node labels are one-hot vectors $\mat{Y} = (\mat{e}_2, \mat{e}_1, \mat{e}_3, \mat{e}_2)^{\top}$. Clearly, we have $\mat{Y} = \hat{\mat{A}} \mat{X}$. We use a GNN model $\GNN(\mat{F}, \mat{A}) = \hat{\mat{A}}\mat{F}$ and a learnable node encoder $\mat{F} = \mat{X} \beta$ to fix $\mat{Y}$. Table \ref{tab:example} shows the learned node features $\mat{F}$ and the accuracy of LD and GLEM. The node labels of nodes $1,4$ are incorrect for the node encoder such that $\mat{F}_{GLEM}$ can not distinguish them. Thus, the GNN model further can not distinguish nodes $1,4$ as they also share the same neighbor. In contrast, inverse labels help LD learn correct node features and further give the correct prediction.
We provide the detailed derivation in Appendix B.2.

\begin{table}[t]
  \centering
  \setlength{\belowcaptionskip}{5pt}
  \caption{The motivating example for LD and GLEM.}
  \label{tab:example}
  \scalebox{1.0}{
    \begin{tabular}{ccc|cccc}
        \toprule
        $\mat{Y}$ & $\mat{X}$ & & $\mat{F}_{GLEM}$ & $\mat{A}\mat{F}_{GLEM}$ & $\mat{F}_{LD}$ & $\mat{A}\mat{F}_{LD}$ \\
        \midrule
        $\mat{e}_2$ & $\mat{e}_1$ & & $\mat{e}_2$ & $(\mat{e}_1 + \mat{e}_3)/2$ & $\mat{e}_1$ & $\mat{e}_2$ \\ 
        $\mat{e}_1$ & $\mat{e}_2$ & & $(\mat{e}_1 + \mat{e}_3)/2$ & $\mat{e}_2$ & $\mat{e}_2$ & $\mat{e}_1$ \\
        $\mat{e}_3$ & $\mat{e}_2$ & & $(\mat{e}_1 + \mat{e}_3)/2$ & $\mat{e}_2$ & $\mat{e}_2$ & $\mat{e}_3$ \\
        $\mat{e}_2$ & $\mat{e}_3$ & & $\mat{e}_2$ & $(\mat{e}_1 + \mat{e}_3)/2$ & $\mat{e}_3$ & $\mat{e}_2$ \\
        \midrule
        \multicolumn{2}{c}{Accuracy} & &  - & 0\% &  - & 100\% \\
        \bottomrule
    \end{tabular}
    }
\end{table}%

\begin{table*}[t]
  \begin{center}
    \caption{Statistics of the datasets in the experiments.
    }\label{tab:datasets}
    \scalebox{1.0}{
    \begin{tabular}{ccccccc}
    \toprule
    \textbf{\#Dataset} & \textbf{\#Node attributes}   & \textbf{\#Metric} &\textbf{Total \#Nodes} & \textbf{Total \#Edges} &  \textbf{\#Avg. Degree} & \textbf{Split Ratio} \\
      \midrule
      ogbn-arxiv & Textural titles and abstracts  & Accuracy  & 169,343 & 1,157,799 & 6.8 & 0.54/0.18/0.28 \\
       ogbn-product & Textural descriptions for products   & Accuracy  & 2,449,029 & 61,859,076 & 25.3 & 0.08/0.02/0.90 \\
       ogbn-protein & Protein sequences  & ROC-AUC  & 132,534 & 39,561,252	& 298.5 & 0.65/0.16/0.19 \\
      \bottomrule
    \end{tabular}
    }
  \end{center}
\end{table*}

\subsection{Theory}

Our theory is based on an observation that the node labels not only depend on node attributes but also graph structures. 
We formulate the observation as the following assumption.
\begin{assumption}\label{ass:attributes_labels}
    The node labels $\mat{Y}$ are given by graph structures and node attributes, i.e., $\mat{Y} = \phi(\hat{\mat{A}}) \psi(\mat{F}^*)$, where $\hat{\mat{A}} = \mat{D}^{-1}\mat{A}$ is the normalized adjacent matrix, $\mat{D}$ is the degree matrix, and $\mat{F}^*_{:,i} = f(\mat{x}_i)$ represents the node feature extracted from the attribute $\mat{a}_i$. Let $\psi: \mathbb{R}^{d_f} \rightarrow \mathbb{R}^{d}$ be an unknown encoder like multi-layer perceptrons (MLPs) and $\phi(\hat{\mat{A}}) = \sum_{n=0}^N \alpha_n^* \hat{\mat{A}}^n$ be an unknown polynomial graph signal filter.
    We assume that $\phi(\hat{\mat{A}})$ is invertible. 
\end{assumption}
We parameterize $\GNN$ by a spectral-based GNN \cite{linear_gnn}, i.e.,
\begin{align*}
    \GNN(\mat{F}^{(\beta)}, \mat{A};\theta) = \phi(\hat{\mat{A}};\theta) \psi(\mat{F}^{(\beta)};\theta).
\end{align*}
Recent works theoretically show that the expressiveness of the spectral-based GNN is equal to the 1-WL test---which bounds the expressiveness of many GNNs  \cite{linear_gnn}. Clearly, the joint training of NEs and GNNs converges to the optimal objective function values, i.e.
\begin{align*}
    \min_{ \theta, \beta} \, \loss(\GNN(\mat{F}^{(\beta)}, \mat{A};\theta), \mat{Y}) = 0.
\end{align*}

We show $\min_{ \theta} \, \loss(\GNN(\mat{F}^{(\beta_{LD})}, \mat{A};\theta), \mat{Y}) = 0$ by the following theorem.

\begin{theorem} \label{thm:cayley}
    If Assumption \ref{ass:attributes_labels} holds, then there exists $\beta_{LD} \in \arg\min_{\beta} \min_{\gamma, \theta} \loss(\psi(\mat{F}^{(\beta)}; \theta),  \mat{Y}^{(\gamma)})$ such that $\min_{ \theta} \, \loss(\GNN(\mat{F}^{(\beta_{LD})}, \mat{A};\theta), \mat{Y}) = 0$.
\end{theorem}
Notably, the motivating example in Section \ref{sec:example} is a counterexample  for $\beta_{GLEM}$. Specifically, in the motivating example, we have $\min_{ \theta} \, \loss(\GNN(\mat{F}^{(\beta_{GLEM})}, \mat{A};\theta), \mat{Y}) > 0$ for all $\beta_{GLEM} \in \arg\min_{\beta} \min_{\theta} \loss(\psi(\mat{F}^{(\beta)}; \theta),  \hat{\mat{Y}})$.

\section{Experiments}\label{sec2}

In this section, we evaluate the performance of LD by comparing it with the state-of-the-art training methods for node classification on multiple large-scale attributed graphs, including citation networks, co-purchase networks, and protein–protein association networks. We run all experiments five times on a single NVIDIA GeForce RTX 3090 (24 GB).

\subsection{Experimental Setups}

\subsubsection{Datasets}

We conduct experiments on ogbn-arxiv, ogbn-product, and ogbn-protein from widely-used Open Graph Benchmark (OGB) \cite{ogb}, whose graphs are citation networks, co-purchase networks, and protein–protein association networks respectively.
OGB provides an official leaderboard\footnote{\url{https://ogb.stanford.edu/docs/leader_nodeprop/}} for a fair comparison of different methods.
They are large-scale (up to 2.4M nodes and 61.9M edges in our experiments) and have been widely used in previous works \cite{gas}.
We follow the realistic training/validation/test splitting methods of OGB, such as by time (ogbn-arxiv), by sales ranking (ogbn-product), and by species which the proteins come from (ogbn-protein).
Their statistics are shown in Table \ref{tab:datasets}.

\subsubsection{Node Attributes}

For ogbn-arxiv and ogbn-product, we use titles/abstracts and textual descriptions for products provided by OGB \cite{ogb} as node attributes.
For ogbn-protein, we map the protein identifiers into the protein sequences based on the STRING database \cite{string}.

\subsubsection{Prediction Tasks and Evaluation Metric}

We use the official evaluation metric provided by OGB \cite{ogb}. Specifically, we use accuracy as the evaluation metric for ogbn-arxiv and ogbn-product, which aim to predict the primary categories of the arxiv papers and the categories of products respectively in a multi-class classification setup.
We use ROC-AUC as the evaluation metric for ogbn-protein, which aims to predict the presence of protein functions in a multi-label binary classification setup.

\subsubsection{Node Encoders and GNN architectures}

For each dataset, we use the state-of-the-art GNN architectures on the OGB leaderboards, including GCN \cite{gcn}, REVGAT \cite{deq_gcn}, GAMLP \cite{gamlp}, SAGN \cite{sagn}, and GAT \cite{gat}. 
For the textual attributes in ogbn-arxiv, following \cite{glem}, we use DeBERTa \cite{deberta} as the node encoder.
For ogbn-product, we follow \cite{glem} to integrate GAMLP with DeBERTa \cite{deberta}, and follow \cite{giant} to integrate SAGN with Bert \cite{bert}.
For the protein sequences in ogbn-protein, we use ESM2 \cite{esm2} as the node encoder.

\subsubsection{Baselines}\label{sec:baseline}

As shown in Section \ref{sec:separate_training}, many existing methods first extract node features $\mat{F}^{(\beta)}$ based on different approximations to Equation \eqref{eqn:obj_ne} and then train GNNs based on the same Equation \eqref{eqn:obj_gnn}.
The node features of our baselines include (1) the features provided by OGB \cite{ogb}, denoted as $\mathbf{X}_{\text {OGB}}$; (2) the embeddings inferred by pre-trained NEs, denoted as $\mathbf{X}_{\text {PNE}}$; (3) the embeddings inferred by fine-tuned NEs with the true node labels, denoted as $\mathbf{X}_{\text {FNE}}$; (4) the GIANT features \cite{giant}, denoted as $\mathbf{X}_{\text {GIANT}}$; (5) the GLEM features \cite{glem}, denoted as $\mathbf{X}_{\text {GLEM}}$.
We summarize the key differences between the above-mentioned baselines and LD in Table \ref{tab:baselines}.
Besides the separate training framework, we also jointly train NEs and GNNs via graph sampling (GAS \cite{gas} on the ogbn-arxiv dataset and SAGE \cite{graphsage} on the ogbn-protein dataset).
To run graph sampling on a single NVIDIA GeForce RTX 3090 (24 GB), we decrease the size of sampled subgraphs as shown in Section \ref{sec:scalable_gnn}. However, SAGE---one of the state-of-the-art neighbor sampling methods---still runs out of GPU memory on the ogbn-protein dataset due to its exponentially increasing complexity.

\begin{table}[t]
  \begin{center}
    \caption{The key differences between the baselines and LD.
    }\label{tab:baselines}
    \scalebox{1.0}{
    \begin{tabular}{cccc}
    \toprule
    \multirow{2}[1]{*}{\textbf{Features}}  & \textbf{Pre-trained} & \textbf{Training with } & \textbf{Training with}   \\
    & \textbf{models} &  \textbf{(pseudo) labels} & \textbf{graph structures} \\
      \midrule
      $\mathbf{X}_{\text {OGB}}$ & $\times$ & $\times$ & $\times$    \\
       $\mathbf{X}_{\text {PNE}}$ & $\checkmark$ & $\times$ & $\times$  \\
       $\mathbf{X}_{\text {FNE}}$ & $\checkmark$  & $\checkmark$ & $\times$ 	 \\
       $\mathbf{X}_{\text {GIANT}}$ & $\checkmark$ & $\times$ & $\checkmark$ 	 \\
       $\mathbf{X}_{\text {GLEM}}$ & $\checkmark$ & $\checkmark$ & $\times$  	 \\
       $\mathbf{X}_{\text {LD}}$ & $\checkmark$ & $\checkmark$ & $\checkmark$  	 \\
      \bottomrule
    \end{tabular}
    }
  \end{center}
\end{table}

\subsubsection{Implementation Details of LD}

We implement the inverse labels $\mat{Y}^{(\gamma)}$ to ensure the label constraints for the multi-class classification and multi-label binary classification setup.
Specifically, we parameterize $\gamma = \softmax(\gamma')$ with a uniform initialization to ensure that the inverse labels are positive and they have a similar norm to the true labels, i.e., $\mat{Y}^{(\gamma)}_{ij} \geq 0$ and $\| \mat{Y}^{(\gamma)} \| \approx \| \mat{Y} \|$. We implement the loss function
\begin{align}\label{eqn:implementation}
    \loss(\mat{H}, \mat{Y}^{(\gamma)}) = l(\mat{H}, (1-\alpha) \mat{Y} + \alpha \normalize(\mat{Y}^{(\gamma)})), 
\end{align}
where $l$ is the cross entropy loss and $\alpha$ is a hyper-parameter to avoid overfitting. 
For the multi-class classification, the normalization function $\normalize(\mat{Y}^{(\gamma)})_{ij} = \mat{Y}^{(\gamma)}_{ij} / (\sum_{k}\mat{Y}^{(\gamma)}_{ik})$ ensures that the output is a probability distribution.
For the multi-label binary classification setup, the normalization function clamps all elements in $\mat{Y}^{(\gamma)}$ into the range $[0,1]$.

We use the PyTorch library \cite{pytorch} to implement label deconvolution.
To ensure a fair comparison, the hyperparameters for pre-retrained NEs and GNNs in experiments are the same as GLEM \cite{glem}.
We follow the implementation of the top-ranked GNNs in the OGB leaderboard.
LD introduces two hyper-parameters $N$ and $\alpha$ defined in Equations \eqref{eqn:inv_label} and \eqref{eqn:implementation} respectively.
We set $N$ to be the number of convolutional layers used in GNNs in experiments.
We set $\alpha=1$ for the GNN models which use the fixed normalized adjacent matrix to aggregate messages from neighbors (e.g., GCN, GAMLP, and SAGN).
We search best $\alpha$ in $\{0.2, 0.5, 1.0\}$ for REVGAT and GAT which use the attention-based aggregation, as we approximate the graph attention by the normalized adjacent matrix inspired by \cite{imp} (see Appendix A for detailed derivation).

For the semi-supervised node classification, we first generate pseudo labels of the nodes in the validation and test sets following GLEM \cite{glem}.
To this end, we train GNNs based on $\mathbf{X}_{\text {PNE}}$.
We do not follow the EM algorithm of GLEM \cite{glem} to iteratively update pseudo labels, as the EM-iteration does not significantly improve performance but suffers from expensive training costs.

\begin{table*}[t]
  \centering
  \setlength{\belowcaptionskip}{5pt}
  \caption{Performance of node classification (mean\,±\,std\%). We bold the best result and underline the second best result. OOM denotes the out-of-memory issue.}
  \label{tab:main_table}

  \scalebox{0.96}{
       \begin{tabular}{c|cl|cccccccc}
    \toprule
    \multicolumn{1}{c}{\textbf{{Datasets}}} & \multicolumn{2}{c}{\textbf{{GNNs}}} & \textbf{Graph Sampling} & $\mathbf{X}_{\text {OGB}}$  & $\mathbf{X}_{\text {GIANT}}$ & $\mathbf{X}_{\text {PNE}}$ & $\mathbf{X}_{\text {FNE}}$ & $\mathbf{X}_{\text {GLEM}}$ & $\mathbf{X}_{\text {LD}}$             \\
    \midrule
    \multirow{4}[2]{*}{ogbn-arxiv} & \multirow{2}[1]{*}{GCN} & \textit{val} & 70.32 ± 0.69  & 73.00 ± 0.17 & 74.89 ± 0.17 & 74.74 ± 0.11 & 76.39 ± 0.11 & 76.86 ± 0.19 & 76.84 ± 0.09  \\
          &       & \textit{test} & 69.29 ± 0.72 & 71.74 ± 0.29 & 73.29 ± 0.10 & 74.04 ± 0.16 & 75.90 ± 0.16 & \underline{75.93 ± 0.19} &\textbf{ 76.22 ± 0.10}  \\
     & \multirow{2}[1]{*}{REVGAT} & \textit{val} & 74.63 ± 0.58  & 75.01 ± 0.10 & 77.01 ± 0.09 & 75.36 ± 0.18 &  75.99 ± 0.18 & 77.49 ± 0.17& 77.62 ± 0.08  \\
          &       & \textit{test} & 74.10 ± 0.52 & 74.02 ± 0.18 & 75.90 ± 0.19 & 75.14 ± 0.08 & 75.52 ± 0.08 & \underline{76.97 ± 0.19} &\textbf{ 77.26 ± 0.17}   \\
    \midrule
    \multirow{4}[2]{*}{ogbn-product} 
          & \multirow{2}[0]{*}{GAMLP} & \textit{val} & ---  & 93.12 ± 0.03 & 93.99 ± 0.04 & 93.21 ± 0.05 & 93.19 ± 0.05 & 94.19 ± 0.01 & 94.15 ± 0.03  \\
          &       & \textit{test} & --- & 83.54 ± 0.09 & 83.16 ± 0.07 & 84.51 ± 0.05 & 82.88 ± 0.05 & \underline{85.09 ± 0.21} &  \textbf{86.45 ± 0.12}  \\

          & \multirow{2}[0]{*}{SAGN} & \textit{val} & ---  & 93.02 ± 0.04 & 93.64 ± 0.05 & 93.25 ± 0.04 & 93.84 ± 0.06  & - & 93.99 ± 0.02  \\
          &       & \textit{test} & --- & 84.35 ± 0.09 & \underline{86.67 ± 0.09} & 85.37 ± 0.08 & 84.81 ± 0.07 & - & \textbf{87.18 ± 0.04} \\ 
           	
    \midrule
    \multirow{2}[2]{*}{ogbn-protein} 
          & \multirow{2}[0]{*}{GAT} & \textit{val} & OOM  &  93.75 ± 0.19 & - & 95.16 ± 0.06 & 95.05 ± 0.12 & - & 95.27 ± 0.07 \\
          &       & \textit{test} & OOM &  88.09 ± 0.16 & - & \underline{88.94 ± 0.14} & 88.66 ± 0.08 & - &\textbf {89.42 ± 0.07} \\
    \bottomrule
    \end{tabular}}%
\end{table*}%

\subsection{Overall performance}

We report the overall performance of LD and baselines in Table \ref{tab:main_table}. Overall, LD outperforms all baselines on the three datasets with different GNN backbones.
As GLEM and GIANT mainly focus on the text-attributed graph, they do not provide the results on the ogbn-protein dataset, whose node attributes are protein sequences. 
The official implementation of GLEM on ogbn-products with SAGN performs much worse than the reported accuracy of 87.36\% on the test data. Therefore, we omit this result in Table \ref{tab:main_table}.

Due to the expensive costs of large NEs, graph sampling uses very small sampled subgraphs to avoid the out-of-memory issue. However, the small subgraphs sacrifice much useful topological information such that their performance is significantly lower than separate training methods.
Moreover, the large sampling error makes their performance lower than $\mathbf{X}_{\text {PNE}}$, which directly trains GNNs via graph sampling with large subgraph sizes given the fixed embeddings inferred by pre-trained NEs.

By noticing that $\mathbf{X}_{\text {PNE}}$ outperforms $\mathbf{X}_{\text {FNE}}$ except for ogbn-arxiv, directly fine-tuning with the true labels may limit the final performance. The phenomenon verifies the motivating example in Table \ref{tab:example} and Assumption \ref{ass:attributes_labels}, i.e., node labels depend on not only node attributes but also graph structures.
Specifically, in ogbn-product, many stores may use product descriptions to attract customers rather than introduce the products, and thus difficult to reflect the product categories.
In ogbn-protein, the edges represent biologically meaningful associations between proteins \cite{ogb} and are important for protein functions.
In ogbn-arxiv, although the abstract and introduction of a scientific paper almost reflect the subject areas of the paper, $\mathbf{X}_{\text {LD}}$ still outperforms $\mathbf{X}_{\text {FNE}}$.
The pseudo-labels introduced by GLEM are difficult to overcome this issue, as the pseudo-labels aim to approximate the true labels. Finally, GIANT aims to introduce graph structure information to node features by self-supervised learning, while it may incorporate much task-irrelevant information due to the neglect of node labels.

\subsection{Runtime}

Another appealing feature of LD is that it is significantly faster than GLEM---the state-of-the-art separate training method---as shown in Table \ref{tab:runtime}. 
First, LD significantly outperforms FNE---which fine-tuned NEs with the true labels---in terms of accuracy, although LD is slightly slower than FNE.
To improve the accuracy of FNE, GLEM integrates pseudos labels with the true labels by an iterative knowledge-distilling process, which requires fine-tuning pre-trained NEs and inferring the whole dataset many times, leading to expensive costs.
In contrast, LD only fine-tunes pre-trained NEs once and infers node features twice as shown in Algorithm \ref{alg:ld}, leading to significantly cheap costs.
Specifically, on the ogbn-product dataset, the accuracy of LD is 1.36\% higher than GLEM, and the runtime is less than one-third of GLEM.

\begin{table}[t]
  \centering
  \caption{The runtime of different separate training methods. Bold font indicates the best result and underlining indicates the second best result. We evaluate the max batch size of LM (max bsz.) and training time on a single 24GB GPU.}
  \label{tab:runtime}
\scalebox{1.0}{
\begin{tabular}{ccccc}
\toprule
\textbf{Datasets}  & \textbf{Metric} & \textbf{FNE} & \textbf{GLEM} & \textbf{LD}  \\   \midrule
\multirow{4}{*}{ogbn-arxiv}     & accuracy   & 75.52  & \underline{75.93}   & \textbf{76.22} \\
                          & max bsz.    & 12 & 12 & 12      \\
                          & time/epoch & \textbf{3150s}        & \underline{4131s}           & 4360s          \\
                          & time/total & \multicolumn{1}{c}{\textbf{2.52h}}       & \multicolumn{1}{c}{6.53h}         & \multicolumn{1}{c}{\underline{3.02h}}       \\    \midrule
\multirow{4}{*}{ogbn-product} & accuracy   & 82.88  & \underline{85.09}  & \textbf{86.45} \\
                          & max bsz.    & 12 & 12 & 12             \\
                          & time/epoch & \textbf{8880s}        & 13133s           & \underline{12871s}          \\
                          & time/total & \textbf{16.88h}      & 75.67h        & \underline{24.78h}      \\   \midrule
\multirow{4}{*}{ogbn-protein} & accuracy   & \underline{88.66}  & -  & \textbf{89.42}  \\
                          & max bsz.    & 4 & - & 4             \\
                          & time/epoch & \textbf{8931s}       &   -         & \underline{12696s}          \\
                          & time/total & \textbf{17.73h}     & -        & \underline{36.22h}      \\   \bottomrule

\end{tabular}}
\end{table}

\subsection{Ablation Study}

Compared with FNE, LD additionally uses pseudo labels on the validation and test data. Thus, we conduct the ablation study in Table \ref{tab:ablation} to show the effectiveness of the inverse labels based on the true labels and the pseudo labels.
Specifically, FNE trains NEs with the original label $\mathbf{Y}_{\mathcal{V}_{train}}$ on the training nodes $\mathcal{V}_{train} \subset \mathcal{V}$, while the inverse label of LD $\hat{\mathbf{Y}}$ are calculated on the whole node $\mathcal{V}$.
As shown in Table \ref{tab:datasets}, the ratio of validation and test nodes is more than 35\%. On the ogbn-product dataset, the number of validation and test nodes is about nine times as the number of training nodes.
Therefore, the pseudo labels and node attributes on vast validation and test nodes (evaluation nodes) may introduce useful information.
To verify that the improvement of LD is mainly due to the effectiveness of the inverse labels rather than the additional information on the evaluation nodes, we introduce $\mathbf{X}_{\text {WO/LD}}$, which fine-tunes pre-trained models with the true labels on the training data and the pseudo labels on the evaluation nodes.

\begin{figure*}[tbp]
	\centering

        \begin{subfigure}{0.33\linewidth}
		\centering
		\includegraphics[width=1.0\linewidth]{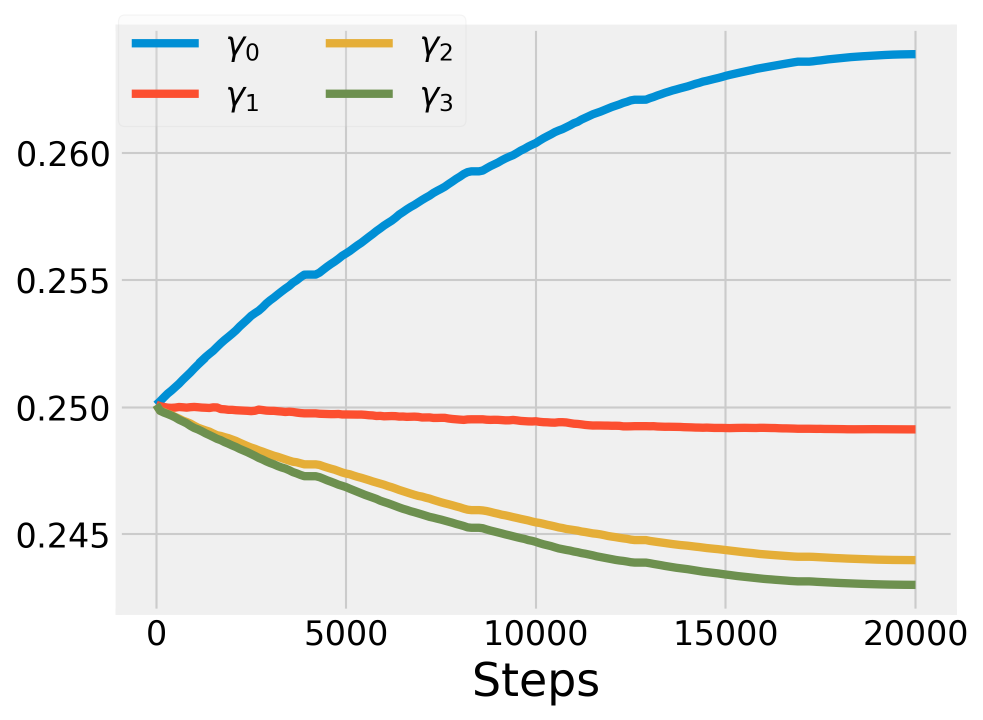}
		\caption{REVGAT \& ogbn-arxiv}
		\label{ogbn-arxiv weight}
	\end{subfigure}
        \begin{subfigure}{0.33\linewidth}
		\centering
		\includegraphics[width=1.0\linewidth]{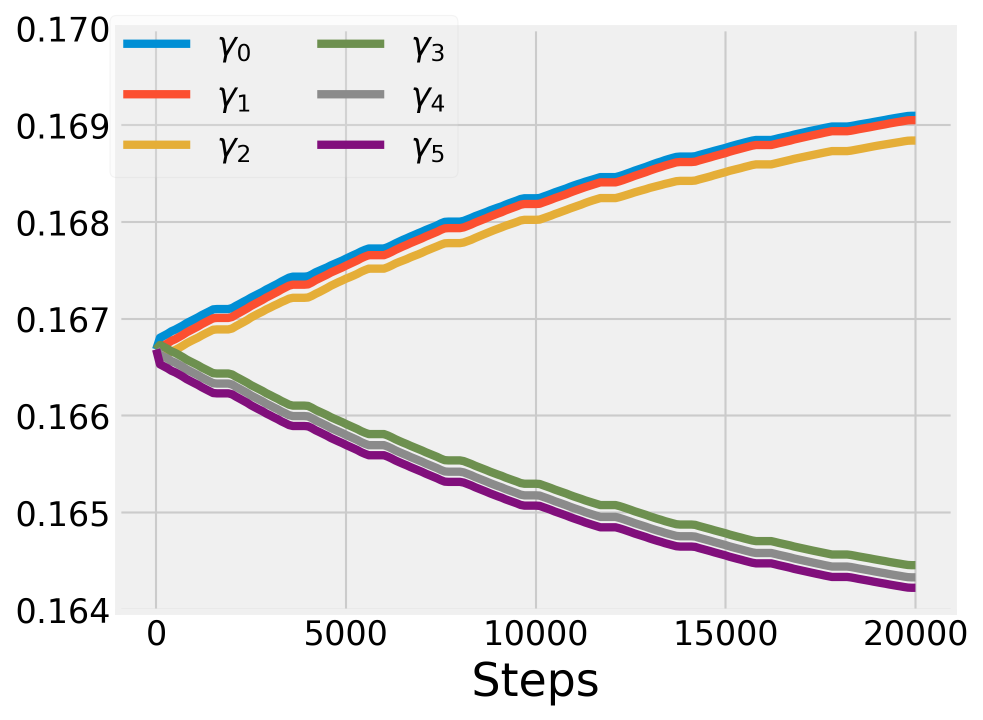}
		\caption{GAMLP \& ogbn-product }
		\label{products weight}
	\end{subfigure}
        \begin{subfigure}{0.33\linewidth}
		\centering
		\includegraphics[width=1.0\linewidth]{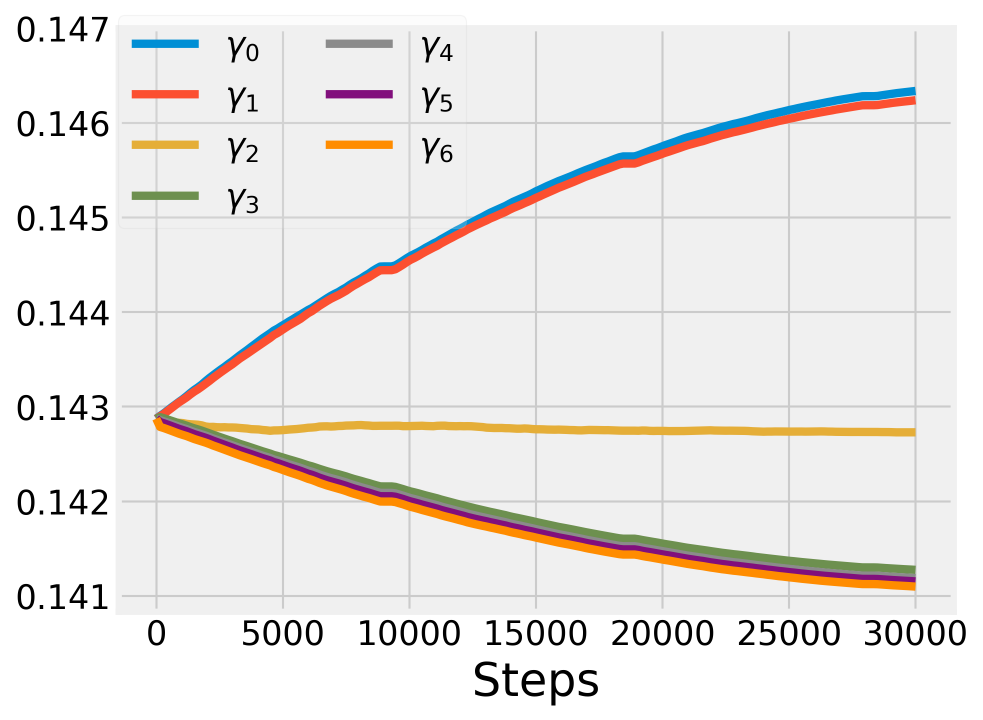}
		\caption{GAT \& ogbn-protein}
		\label{ogbn-protein weight}
	\end{subfigure}
        \caption{Training curves of weights $\gamma_i$.}


    \label{fig:weight}
\end{figure*}

\begin{figure*}[tbp]
	\centering

        \begin{subfigure}{0.33\linewidth}
		\centering
		\includegraphics[width=1.0\linewidth]{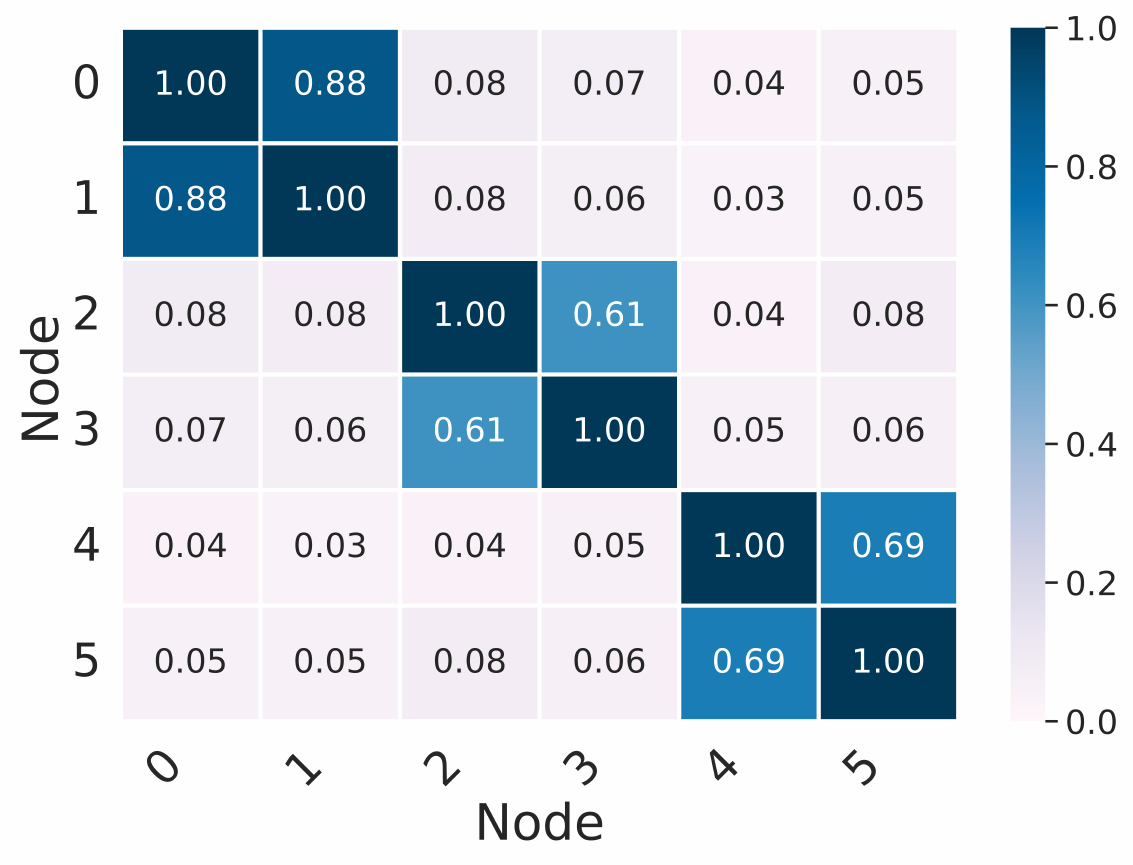}
		\caption{Text similarity}
		\label{fig:Text similarity}
	\end{subfigure}
        \begin{subfigure}{0.33\linewidth}
		\centering
		\includegraphics[width=1.0\linewidth]{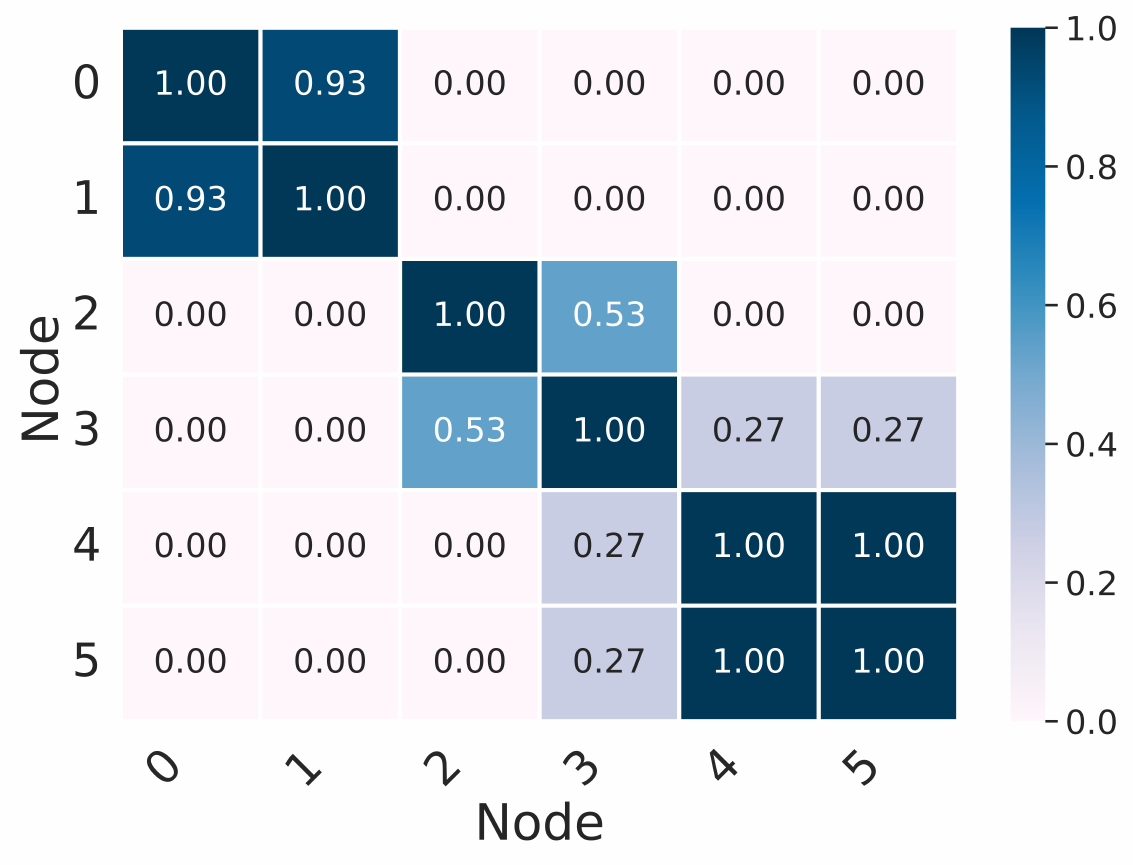}
		\caption{Label similarity of the inverse labels}
		\label{fig:Inverse label similarity}
	\end{subfigure}
        \begin{subfigure}{0.33\linewidth}
		\centering
		\includegraphics[width=1.0\linewidth]{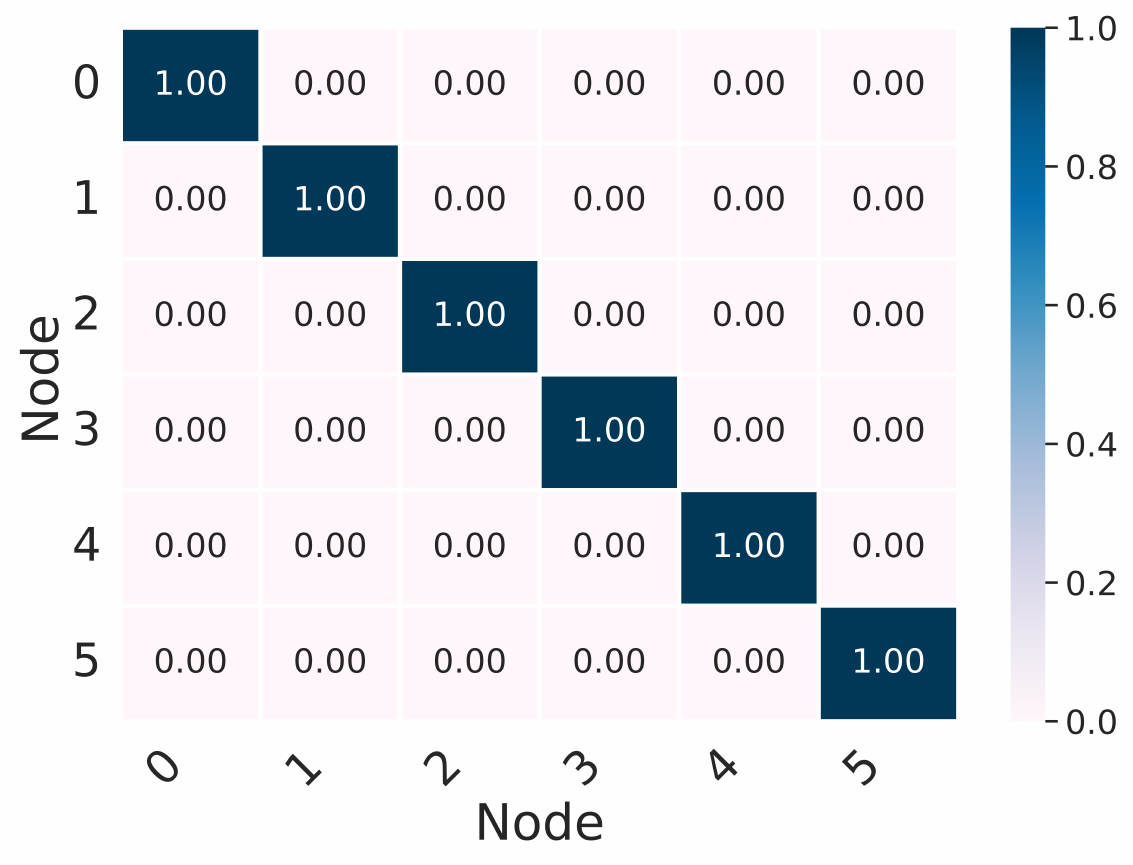}
		\caption{Label similarity of the true labels}
		\label{fig:Original label similarity}
	\end{subfigure}
        \caption{Text similarity and label similarity.}
    \label{fig:similarity}

\end{figure*}

We set $\alpha=0$ in Equation \eqref{eqn:implementation} to implement $\mathbf{X}_{\text {WO/LD}}$.
Specifically, $\mathbf{X}_{\text{WO/LD}}$ replaces the loss function of LD  in Equation \eqref{eqn:implementation} with the $\loss(\mathbf{H}, \mathbf{Y})$, where $\mathbf{Y}$ consists of the original labels on the training nodes and pseudo labels on the evaluation nodes. Then, we use the modified loss function in the separate training framework of LD, i.e., Line 10 in Algorithm \ref{alg:ld}.
Other hyper-parameters and training pipelines are the same as that of $\mathbf{X}_{\text {LD}}$.
We report the mean and standard deviation of the test accuracy across five runs with different random seeds.
in Table \ref{tab:ablation}.
The results demonstrate that the pseudo labels slightly improve the prediction performance against $\mathbf{X}_{\text {FNE}}$ based on the true labels, and $\mathbf{X}_{\text {LD}}$ significantly improves the prediction performance on all datasets against $\mathbf{X}_{\text {WO/LD}}$.
Specifically, on the ogbn-product dataset, $\mathbf{X}_{\text {WO/LD}}$ is approximately 0.95\% higher than $\mathbf{X}_{\text {FNE}}$, while $\mathbf{X}_{\text {LD}}$ is significantly 2.62\% higher than $\mathbf{X}_{\text {WO/LD}}$.

\begin{table}[t]
  \centering
  \caption{Results of ablation study.}
  \label{tab:ablation}
\scalebox{1.0}{
\begin{tabular}{ccccc}
    \toprule 
                       \textbf{{Datasets}} & \textbf{{GNNs}}    & $\mathbf{X}_{\text {FNE}}$   & $\mathbf{X}_{\text {WO/LD}}$  & $\mathbf{X}_{\text {LD}}$ \\ 
    \midrule
    \multirow{1}{*}{ogbn-arxiv} & \multirow{1}{*}{REVGAT}  &  75.52±0.08 & 75.35±0.12 & \textbf{77.26±0.17} \\ 
    \midrule
    \multirow{1}{*}{ogbn-product} & \multirow{1}{*}{GAMLP}  & 82.88±0.05 & 83.83±0.18 & \textbf{86.45±0.12} \\ 
    \midrule
    \multirow{1}{*}{ogbn-protein}  & \multirow{1}{*}{GAT}    & 88.66±0.08 & 89.15±0.12 & \textbf{89.42±0.07} \\ 
    \bottomrule
\end{tabular}}
\end{table}

\begin{figure*}[tbp]
	\centering

        \begin{subfigure}{0.33\linewidth}
		\centering
		\includegraphics[width=1.0\linewidth]{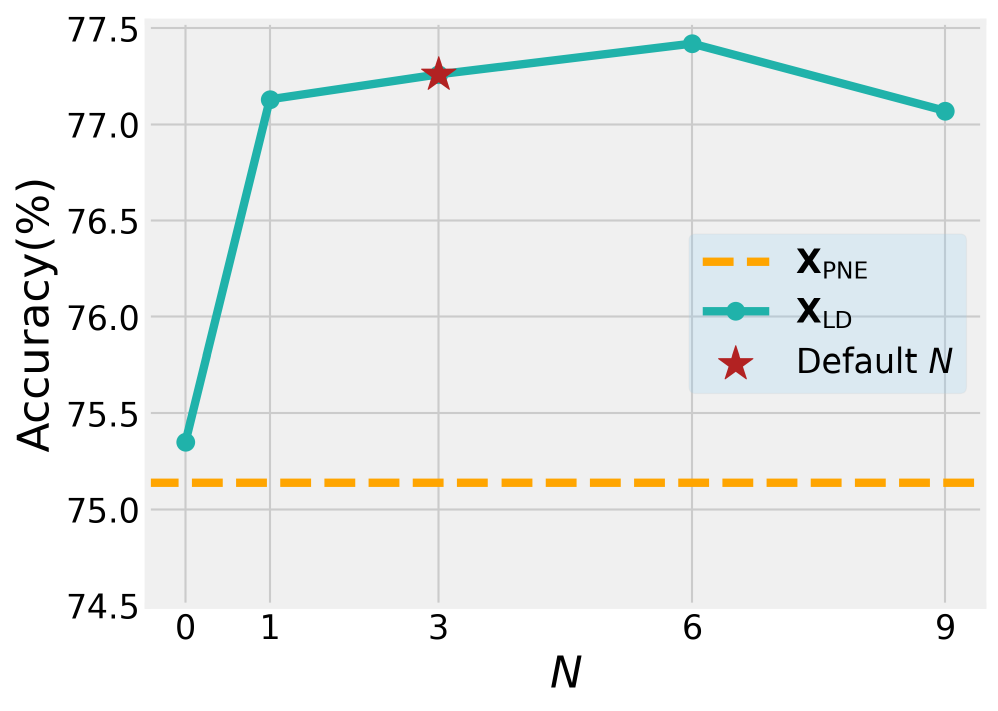}
		\caption{REVGAT \& ogbn-arxiv}
		\label{n-layer-arxiv}
	\end{subfigure}
        \begin{subfigure}{0.33\linewidth}
		\centering
		\includegraphics[width=1.0\linewidth]{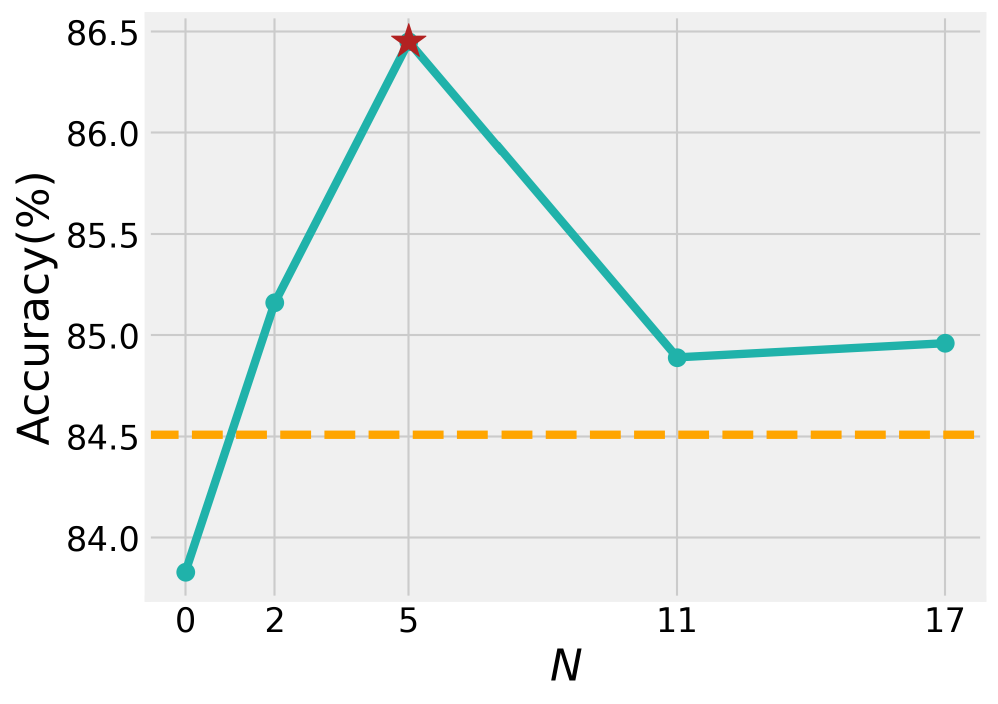}
		\caption{GAMLP \& ogbn-product }
		\label{n-layer-products}
	\end{subfigure}
        \begin{subfigure}{0.33\linewidth}
		\centering
		\includegraphics[width=1.0\linewidth]{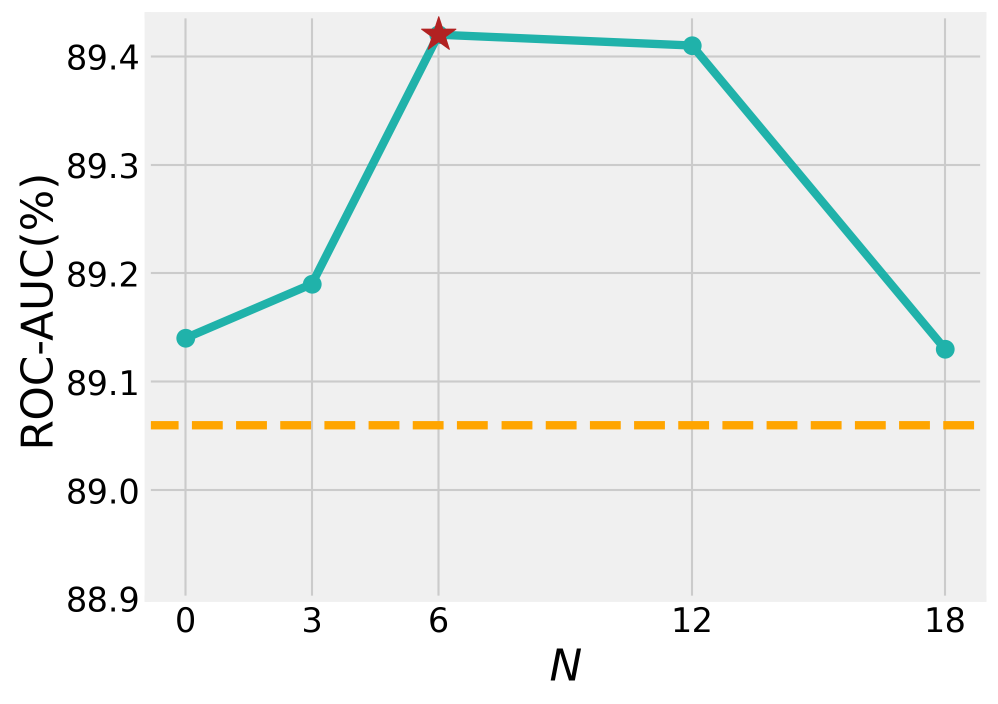}
		\caption{GAT \& ogbn-protein}
		\label{n-layer-proteins}
	\end{subfigure}
        \caption{The performance of node classification under varying values of label deconvolution layers $N$. The used GNN layers are 3, 5, and 6 on the ogbn-arxiv, ogbn-product, and ogbn-protein, respectively.}

    \label{fig:n-layer}
\end{figure*}

\begin{figure*}[tbp]
	\centering

        \begin{subfigure}{0.33\linewidth}
		\centering
		\includegraphics[width=1.0\linewidth]{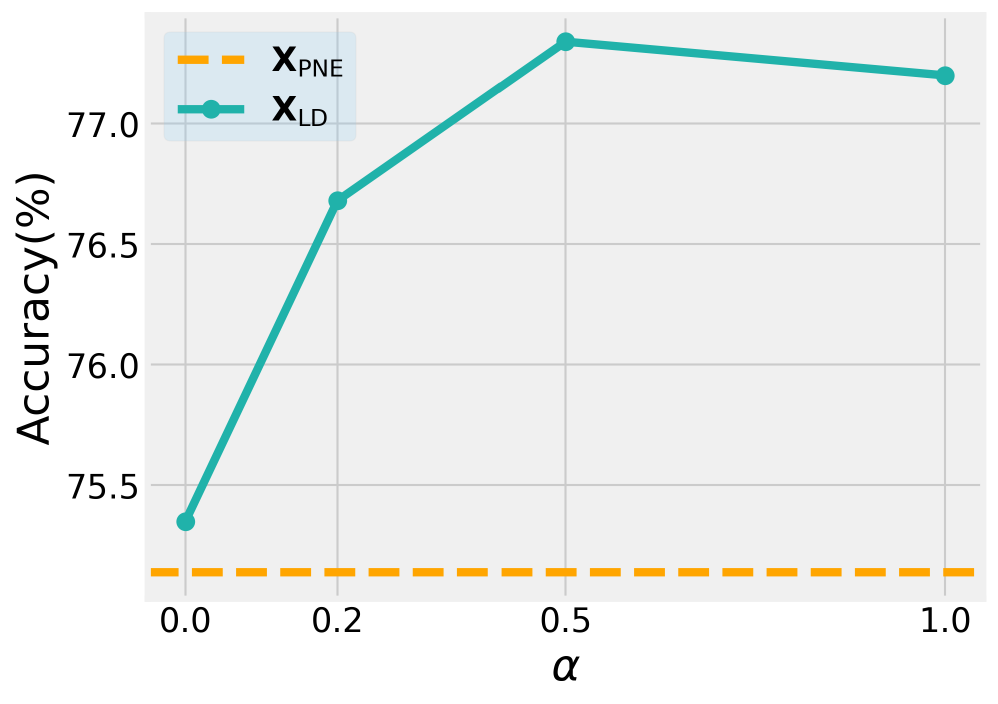}
		\caption{REVGAT \& ogbn-arxiv}
		\label{alpha_arxiv}
	\end{subfigure}
        \begin{subfigure}{0.33\linewidth}
		\centering
		\includegraphics[width=1.0\linewidth]{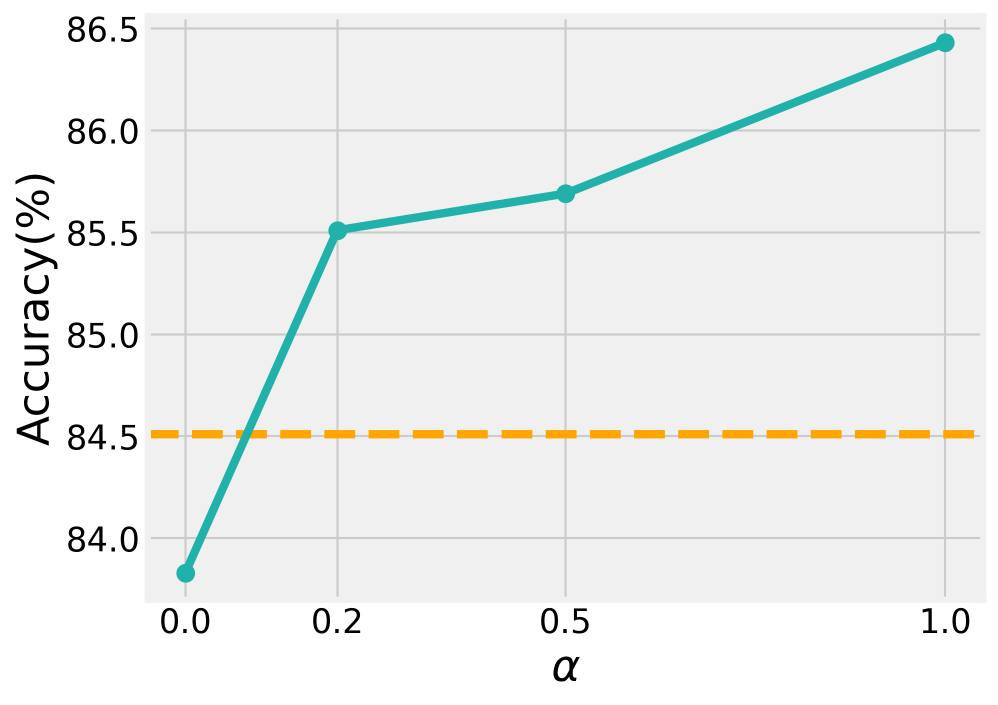}
		\caption{GAMLP \& ogbn-product }
		\label{alpha_products}
	\end{subfigure}
        \begin{subfigure}{0.33\linewidth}
		\centering
		\includegraphics[width=1.0\linewidth]{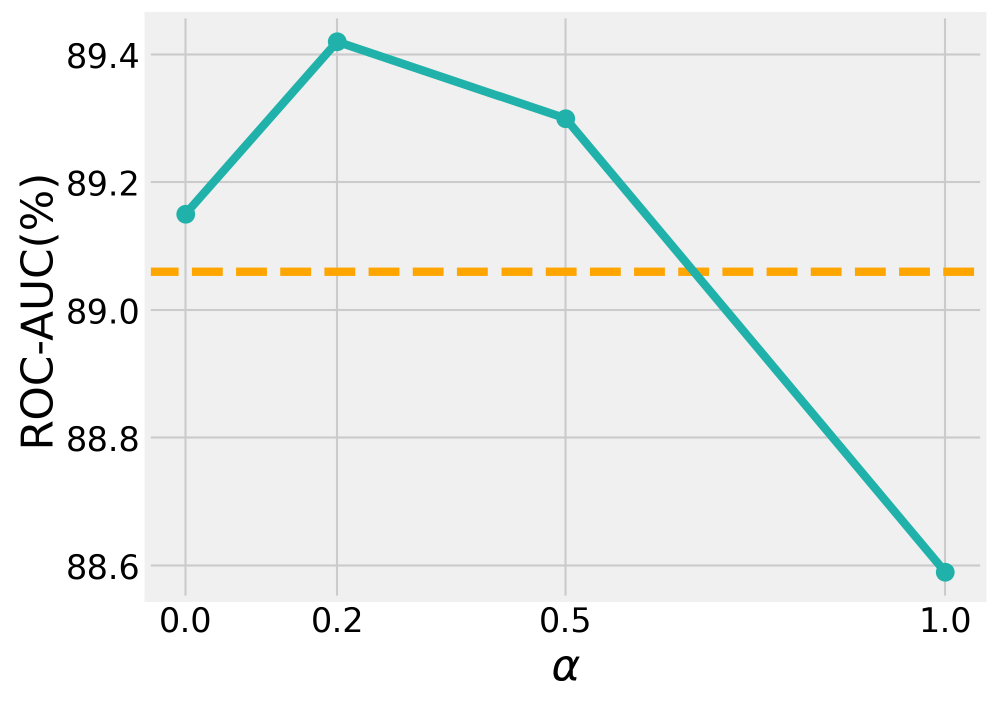}
		\caption{GAT \& ogbn-protein}
		\label{alpha_proteins}
	\end{subfigure}
        \caption{The performance of node classification under varying values of $\alpha$. }

    \label{fig:alpha}
\end{figure*}

\subsection{Analysis of Inverse Labels}

In this section, we empirically compare the difference between the inverse labels and the true labels.

From Equation \eqref{eqn:inv_label}, the inverse labels $\mat{Y}^{(\gamma)}$ are a weighted sum of the true labels and their $i$-hops neighbors' labels. We thus plot the weights $\gamma_i$ during the fine-tuning process in Fig. \ref{fig:weight}.
The inverse labels tend to be the true labels or the $i$-hops neighbors' labels with small $i$.
This is because the true labels and the $i$-hops neighbors' labels with small $i$ are still the most important supervision signals among all hops' labels for node classification.
Moreover, the $i$-hops neighbors' labels with large $i$ suffer from the over-smoothing issue, i.e., the $i$-hops neighbors' labels may tend to be indistinguishable as $i$ increases \cite{over-smoothing}.
Notably, the weight $\gamma_i$ does not converge to a trivial solution where
$\gamma_0 \rightarrow 1, \gamma_i \rightarrow 0 , i \geq 1$.
This implies that other hops' labels are helpful to node feature extraction and label deconvolution effectively alleviates the label noise from the graph structures.

In order to further compare the inverse labels and the true labels, we show the similarity of the node attributes and the similarity of labels in Fig. \ref{fig:similarity}.
We randomly selected several pairs of nodes from the ogbn-arxiv dataset with highly similar texts (i.e., the text similarity is more than 0.6) but different labels (nodes 0 and 1, 2 and 3, 4 and 5).
Following \cite{giant}, we use the TF-IDF algorithm and the cosine similarity to evaluate the text similarity and the label similarity, respectively.
Fig. \ref{fig:Text similarity} each pair of nodes has high similarity, but the nodes in different pairs have low similarity as we select them independently.
Fig.s \ref{fig:Inverse label similarity} and \ref{fig:Original label similarity} show that the inverse labels provide similar supervision signals for the nodes with similar texts and different supervision signals for the nodes with different texts. However, the true labels fail.
Therefore, the inverse labels preserve the attribute semantics by reducing label noise in the graph structures.

\subsection{Hyper-parameter Sensitivity}

The additional hyper-parameters of LD contain $N$ and $\alpha$ defined in Equations \eqref{eqn:inv_label} and \eqref{eqn:implementation} respectively.
We provide the performance curves on different datasets for $N$ and $\alpha$.
When exploring the effect of a hyper-parameter, we fix the other as their best values.

As shown in Fig. \ref{fig:n-layer}, LD achieves the best performance when $N$ is equal to the number of convolution layers of GNNs on the ogbn-product and ogbn-protein datasets.
On the ogbn-arxiv dataset, the performance at $N=6$ is slightly better than that at $N=3$, which is the number of the convolution layers of GNNs.
Indeed, $N$ highly impacts the expressiveness of label deconvolution.
The inverse labels under large $N$ increase the expressiveness while may suffer from overfitting.

As shown in Fig. \ref{fig:alpha}, LD achieves the best performance with large $\alpha$ on the ogbn-arxiv and ogbn-product datasets.
Moreover, the performance of LD with $\alpha=1$ is similar to the best performance on these datasets.
LD with the large $\alpha$ indicates that the inverse labels dominate the training behavior of NEs.
However, on the ogbn-arxiv dataset, LD achieves the best performance with a small $\alpha=0.2$.
As the average degree of ogbn-protein is significantly larger than ogbn-arxiv and ogbn-product (see Table \ref{tab:datasets}), the $i$-hop labels $\mat{K}_i = \hat{\mat{A}}^i \mat{Y}$ under the large degree are easily become indistinguishable, leading to the over-smoothing issue \cite{over-smoothing}.
Despite the over-smoothing issue, LD with $\alpha=0.2$ still outperforms LD with $\alpha=0.0$, which only uses the true labels.
Overall, the inverse labels of LD significantly improve the prediction performance.

\subsection{Scaling Effect for Node Encoders}

Many fields have observed a scaling law in many large models, where performance on downstream tasks improves as more parameters are used. To investigate whether The large-scale node encoders of LD exhibit a similar property, we evaluate the performance of LD in terms of different NE sizes.
Specifically, on the ogbn-arxiv dataset, we integrate RevGAT with $\text{DeBERTa-base}_{\text{100M}}$, $\text{DeBERTa-large}_{\text{350M}}$, and $\text{DeBERTa-xlarge}_{\text{750M}}$. On the ogbn-proetin dataset, we integrate GAT with $\text{ESM2}_{\text{8M}}$, $\text{ESM2}_{\text{35M}}$, $\text{ESM2}_{\text{150M}}$, and $\text{ESM2}_{\text{650M}}$.
We evaluate our proposed LD, $\mathbf{X}_{\text{PNE}}$ and $\mathbf{X}_{\text{FNE}}$ introduced in Section \ref{sec:baseline}.
$\mathbf{X}_{\text {PNE}}$ directly trains a GNN with the embeddings inferred by pre-trained NEs and $\mathbf{X}_{\text{FNE}}$ trains a GNN with the LM embeddings inferred by fine-tuned NEs with the true node labels.

Fig. \ref{fig:scaling} shows the scaling effect for large node encoders (NEs). 
The performance of most models improves as more parameters are used in large NEs.
Notably, as the number of parameters increases, the performance improvement of $\mathbf{X}_{\text {LD}}$ is more significant than $\mathbf{X}_{\text {PNE}}$ and $\mathbf{X}_{\text {FNE}}$, showing the appealing scalability.
Therefore, although more parameters of NEs benefit many models, the learning bias of existing methods diminishes their growth rates.

\begin{figure*}[tbp]
	\centering

    \begin{subfigure}{0.48\linewidth}
		\centering
		\includegraphics[width=1.0\linewidth]{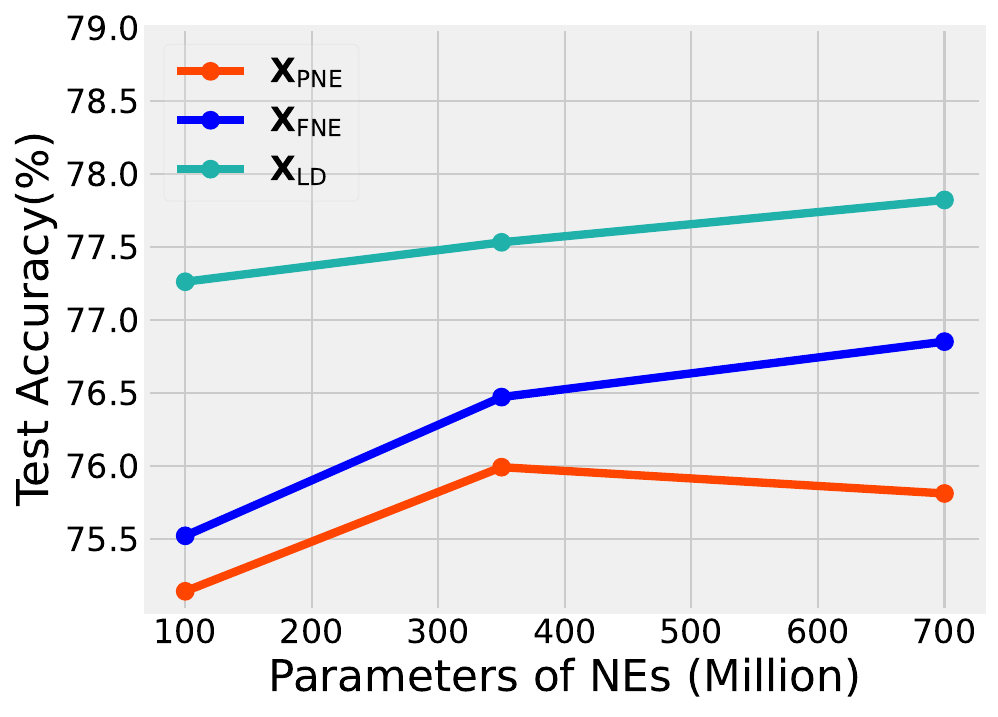}
		\caption{Deberta \& ogbn-arxiv}
		\label{fig:scaling_arxiv}
	\end{subfigure}
    \begin{subfigure}{0.48\linewidth}
		\centering
		\includegraphics[width=1.0\linewidth]{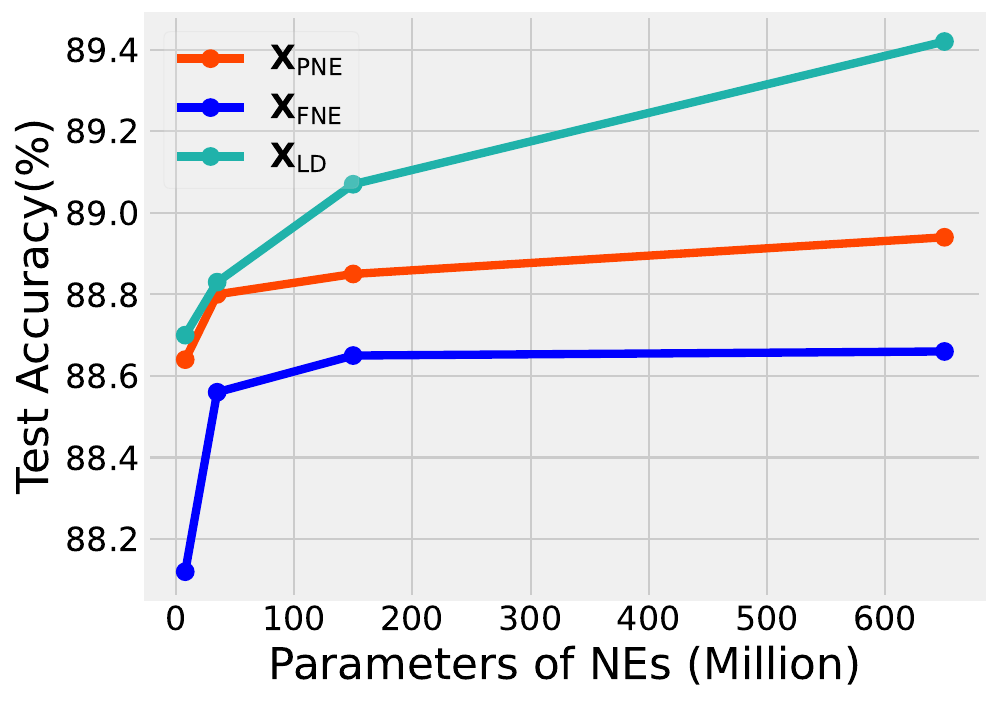}
		\caption{ESM2 \& ogbn-protein}
		\label{fig:scaling_proteins}
	\end{subfigure}
        \caption{The accuracy under different parameters of large NEs.}

    \label{fig:scaling}
\end{figure*}

\subsection{Inductive Node Classification}

Besides the transductive node classification, inductive settings---where test nodes are unseen during the training phase---are also essential in practice.
For example, the papers next year (new test nodes) are unseen in the citation network.

\udfsection{Baselines.} LD applies to the inductive setting. However, the baseline $\mathbf{X}_{\text{GLEM}}$ \cite{glem} is the same as the $\mathbf{X}_{\text{FNE}}$ (the LM embeddings inferred by fine-tuned NEs with the true node labels), as the pseudo-label distillation for test nodes is unavailable in the inductive setting. We also compare LD with the embeddings inferred by pre-trained NEs (denoted as $\mathbf{X}_{\text{PNE}}$).

\udfsection{Datasets.} As many inductive benchmarks do not contain node attributes such as textual content, we modify the ogbn-arxiv and ogbn-protein datasets into the inductive setting following \cite{graphsage, eerm}.
We remove all validation and test nodes from the whole graph to generate the training graph. The evaluation graph is still the whole graph.
In the ogbn-arxiv dataset, the training, validation, and test nodes are the papers published until 2017, those published in 2018, and those published since 2019 respectively.
Under the inductive setting, the model learns from the historical papers and then predicts the properties of the unseen new papers.
In the ogbn-protein dataset, the training, validation, and test protein nodes correspond to different species which the proteins come from.
Under the inductive setting, the model learns from a given species and then predicts the properties of the proteins from the unseen species.

\udfsection{Performance.} We report the inductive performance of LD and baselines in Table \ref{tab:inductive}. Overall, LD outperforms all baselines on the ogbn-arxiv and ogbn-protein datasets by a large margin. 
The performance of $\mathbf{X}_{\text {GLEM}}$---one of the state-of-the-art separate training methods---is similar to that of $\mathbf{X}_{\text {PNE}}$, as its key contribution is a pseudo-label distillation framework, where the inductive setting does not contain the pseudo labels of the evaluation nodes during the training phase.
Thus, compared with other separate training methods, another appealing feature of LD is that LD is able to apply to the inductive setting.

\begin{table}[]
  \centering
  \setlength{\belowcaptionskip}{5pt}
  \caption{Performance of inductive node classification. We bold the best result.} 
  \label{tab:inductive}

  \scalebox{1.10}{
    \begin{tabular}{c|cl|cccc}
        \toprule
        \multicolumn{1}{c}{\textbf{{Datasets}}} & \multicolumn{2}{c}{\textbf{{GNNs}}} & $\mathbf{X}_{\text {PNE}}$ &  $\mathbf{X}_{\text {GLEM}}$ & $\mathbf{X}_{\text {LD}}$             \\
        \midrule
        \multirow{2}[1]{*}{ogbn-arxiv}  & \multirow{2}[1]{*}{REVGAT} & \textit{val}   &  75.52 & 76.25 & 77.02  \\
              &       & \textit{test}  & 75.23 & 75.49 &\textbf{ 76.83}   \\
        \midrule
        \multirow{2}[1]{*}{ogbn-protein}  & \multirow{2}[0]{*}{GAT} & \textit{val}   & 95.19  & 95.37 & 95.53 \\
              &       & \textit{test}  & 88.75 & 88.72 &\textbf {89.12} \\
        \bottomrule
    \end{tabular}}%
\end{table}%

\section{related work}

%

\subsection{Training Methods for large NEs and GNNs}

An ideal idea is to jointly train large NEs and GNNs to simultaneously optimize their parameters.
However, due to the severe scalability issue of GNNs \cite{graphsage, lmc, graphsaint} and the excessive model complexity of pre-trained models, the joint training often runs out of the GPU memory.
To avoid the out-of-memory issue, some methods \cite{textgnn, text_level_gnn} restrict the feature convolutions to very small sampled subgraphs, severely hurting topological structures in graphs.

To encode graph structures and node attributes, many studies propose to separately train NEs and GNNs.
For example, some methods \cite{giant, linkbert} first propose scalable self-supervised learning to train NEs---which encode node attributes into node features---and then integrate node features with graph structures by GNNs. 
The proposed self-supervised tasks aim to incorporate graph topological information into node features.
However, it is unclear whether the additional graph topological information is helpful to GNNs, which encode similar information.
Besides, GLEM \cite{glem} proposes an iterative pseudo-label distillation framework, which iteratively trains NEs and GNNs.
The pseudo-label distillation framework aims to improve the quality of the pseudo labels, which is orthogonal to our proposed LD.

Notably, many existing training methods mainly focus on the text-attributed graph (TAG) \cite{graphformer} whose node attributes are texts. An appealing feature of LD is that it applies to general node attributes including protein sequences.

\subsection{Scalable Graph Neural Networks}

To avoid full-batch inference and training on the large-scale graphs, the graph sampling techniques run GNNs on sample small subgraphs.
We follow \cite{dlg} to categorize these methods into node, layer, and subgraph-wise sampling.
Node-wise sampling methods \cite{graphsage, vrgcn, graphfm} recursively sample a fixed number of neighbors to construct different subgraphs for different GNN layers.
Although they decrease the bases in the exponentially increasing sizes of the subgraph with the number of GNN layers, they still suffer from the neighbor explosion issue.
To tackle this issue, layer-wise sampling methods \cite{fastgcn, ladies, adapt} use importance sampling to recursively sample different subgraphs with the same size for different GNN layers.
However, they suffer from sparse connections between sampled nodes, leading unstable training process.
To improve stability and accelerate convergence, subgraph-wise sampling methods \cite{cluster_gcn, graphsaint, gas, lmc} sample the same subgraph with more connections for different GNN layers.
Despite their success, reducing the subgraph sizes significantly sacrifices the model performance.
In practice, they are usually applicable to shallow GNNs with large subgraph sizes rather than deep pre-trained node encoders and GNNs with very small subgraph sizes \cite{glem}.

Another line of related work is to design scalable GNNs by moving feature convolutions into the pre-processing process \cite{sgc, ssgc, appnp, gamlp, linear_gnn}.
They first perform feature convolutions for the initial node features and save them as additional node features in the pre-processing process.
Then, we design various node classifiers to learn node representations without message passing.
JacobiConv \cite{linear_gnn} shows that the expressiveness of the pre-convolution architectures is the same as 1-WL under some mild assumptions for node classification.
However, the pre-processing scheme is not applicable to the learned node features by node encoders.
Label deconvolution is a novel and effective pre-processing scheme to extract node features and hence it fills this gap.

\section{conclusion}

In this paper, we propose an efficient and effective label regularization technique, namely Label Deconvolution (LD), to alleviate the learning bias relative to the joint training.
we show that LD converges to the optimal objective function values by the joint training under some mild assumptions.
 Extensive experiments on Open Graph Benchmark datasets demonstrate that LD outperforms state-of-the-art training methods in terms of prediction performance and efficiency by a large margin.

\nocite{ngnn}

\ifCLASSOPTIONcompsoc
  \section*{Acknowledgments}
\else
  \section*{Acknowledgment}
\fi

The authors would like to thank the associate editor and all the anonymous reviewers for their insightful comments. 
This work was supported in part by National Key R\&D Program of China under contract 2022ZD0119801, National Nature Science Foundations of China grants U23A20388, 62021001, U19B2026, and U19B2044. This work was supported in part by the Alibaba Group as well.


%

\ifCLASSOPTIONcompsoc
  \section*{\modifyok{}{Acknowledgments}}
\else
  \section*{Acknowledgment}
\fi

The authors would like to thank all the anonymous reviewers for their insightful comments.
This work is supported by National Key R\&D Program of China under contract 2022ZD0119801.
This work is also supported in part by National Nature Science Foundations of China grants U19B2026, U19B2044, 61836011, 62021001, and 61836006.

\bibliographystyle{IEEEtran}
\bibliography{main}

\ifCLASSOPTIONcaptionsoff
  \newpage
\fi

\begin{IEEEbiography}[{\includegraphics[width=1in,height=1.25in,clip,keepaspectratio]{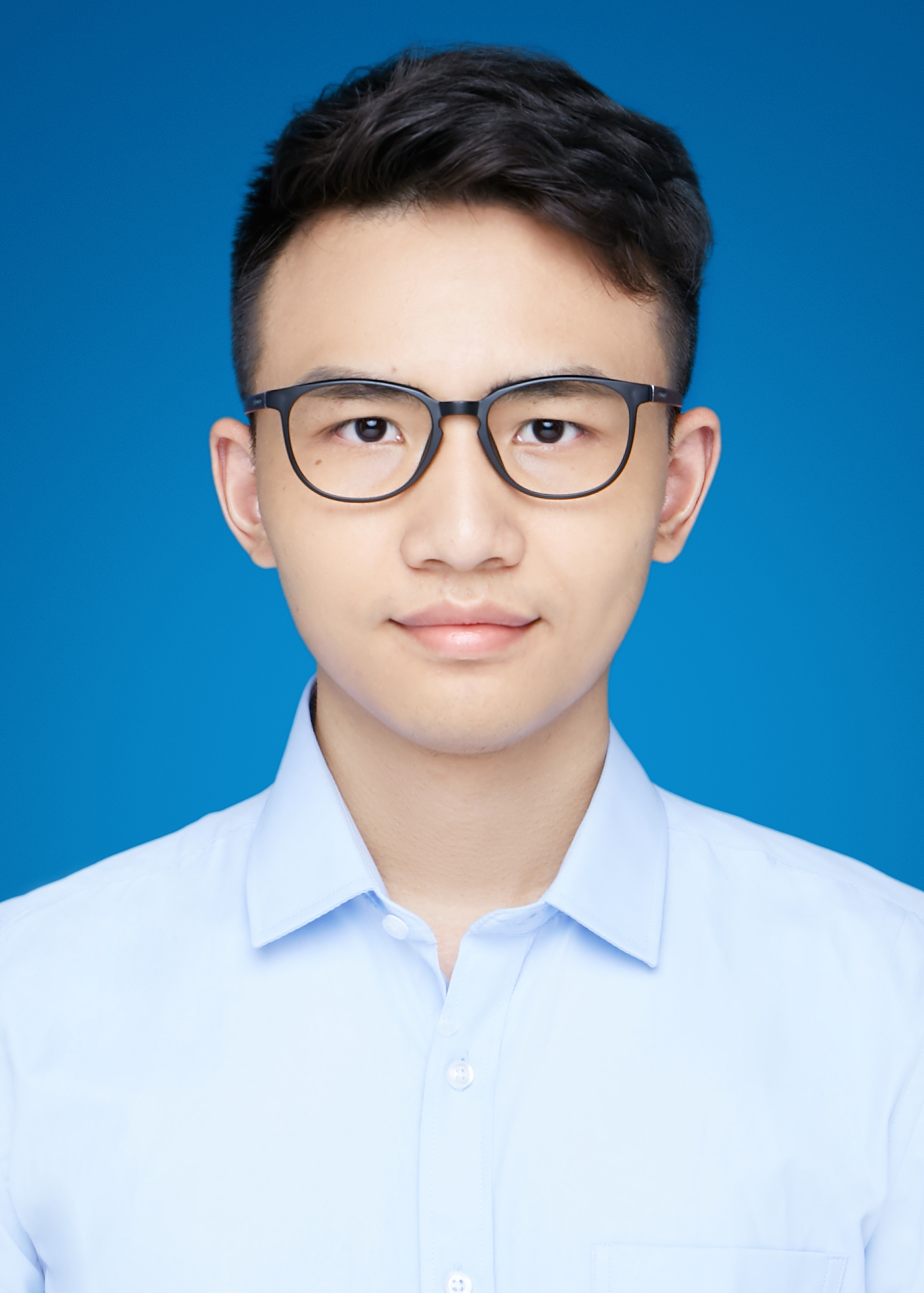}}]{Zhihao Shi}
  received the B.Sc. degree in Department of Electronic Engineering and Information Science from University of Science and Technology of China, Hefei, China, in 2020. a Ph.D. candidate in the Department of Electronic Engineering and Information Science at University of Science and Technology of China, Hefei, China. His research interests include graph representation learning and natural language processing.
\end{IEEEbiography}

\begin{IEEEbiography}[{\includegraphics[width=1in,height=1.25in,clip,keepaspectratio]{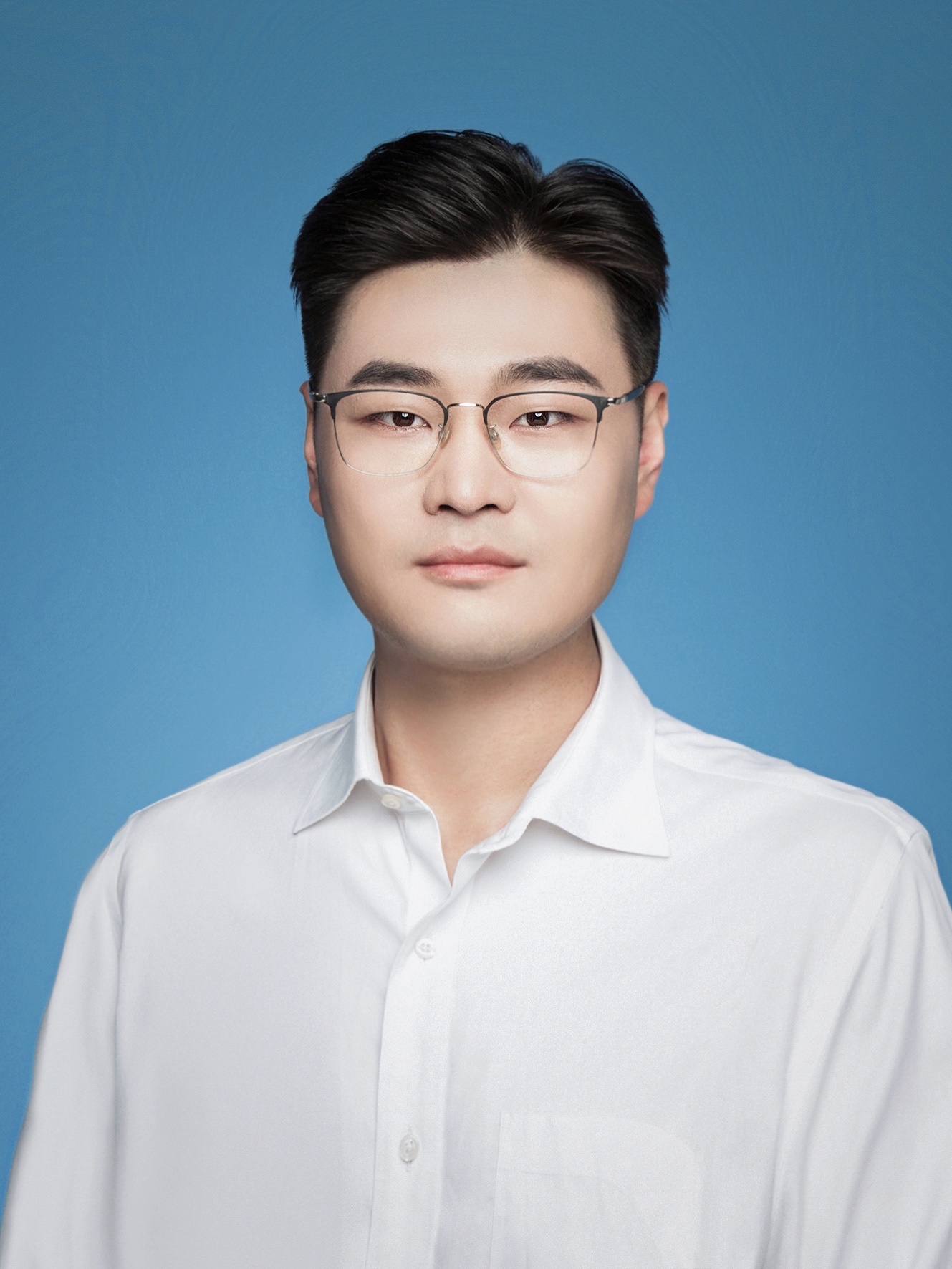}}]{Jie Wang}  (Senior Menber, IEEE) is currently a professor in the Department of Electronic Engineering and Information Science at University of Science and Technology of China, Hefei, China.
He received the B.Sc. degree in electronic information science and technology from University of Science and Technology of China, Hefei, China, in 2005, and the Ph.D. degree in computational science from the Florida \mbox{State} University, Tallahassee, FL, in 2011.
Before joining USTC, Dr. Wang held a position of research assistant professor at University of Michigan from 2015.
His research interests include reinforcement learning, knowledge graph, large-scale optimization, deep learning, etc.
He has published many papers on top machine learning and data mining journals and conferences such as JMLR, TPAMI, NIPS, ICML, and KDD.
He is a senior member of IEEE.
He has served as an associate editor for Neurocomputing and an editorial board member of Data Mining and Knowledge Discovery.
\end{IEEEbiography}

\begin{IEEEbiography}[{\includegraphics[width=1in,height=1.25in,clip,keepaspectratio]{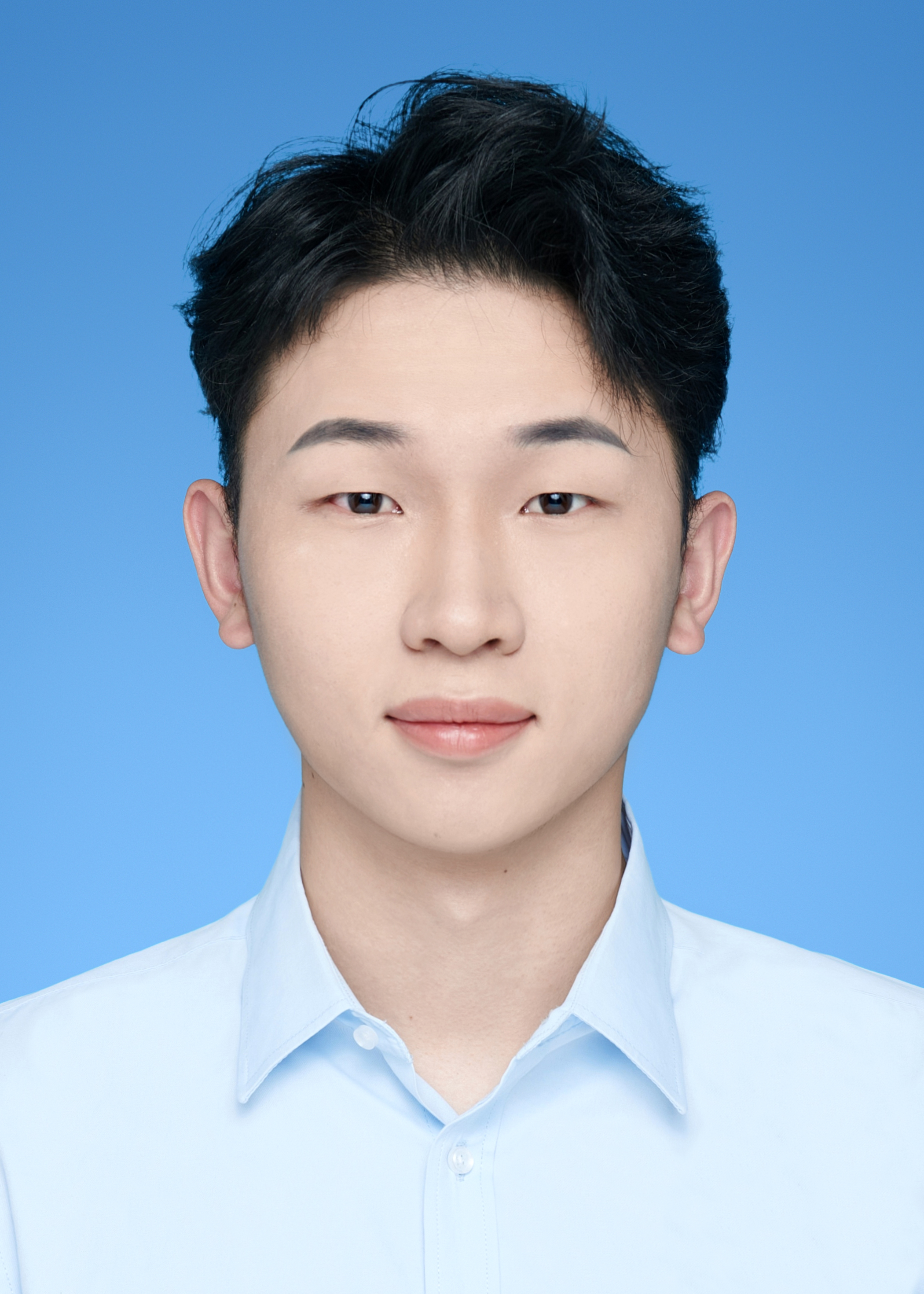}}]{Fanghua Lu}
  received the B.E degree in Mechanical Design and Automation from Shanghai University, Shanghai, China, in 2023. He is currently a graduate student in the Department of Electronic Engineering and Information Science at the University of Science and Technology of China. His research interests include graph learning and natural language processing.
\end{IEEEbiography}

\begin{IEEEbiography}[{\includegraphics[width=1in,height=1.25in,clip,keepaspectratio]{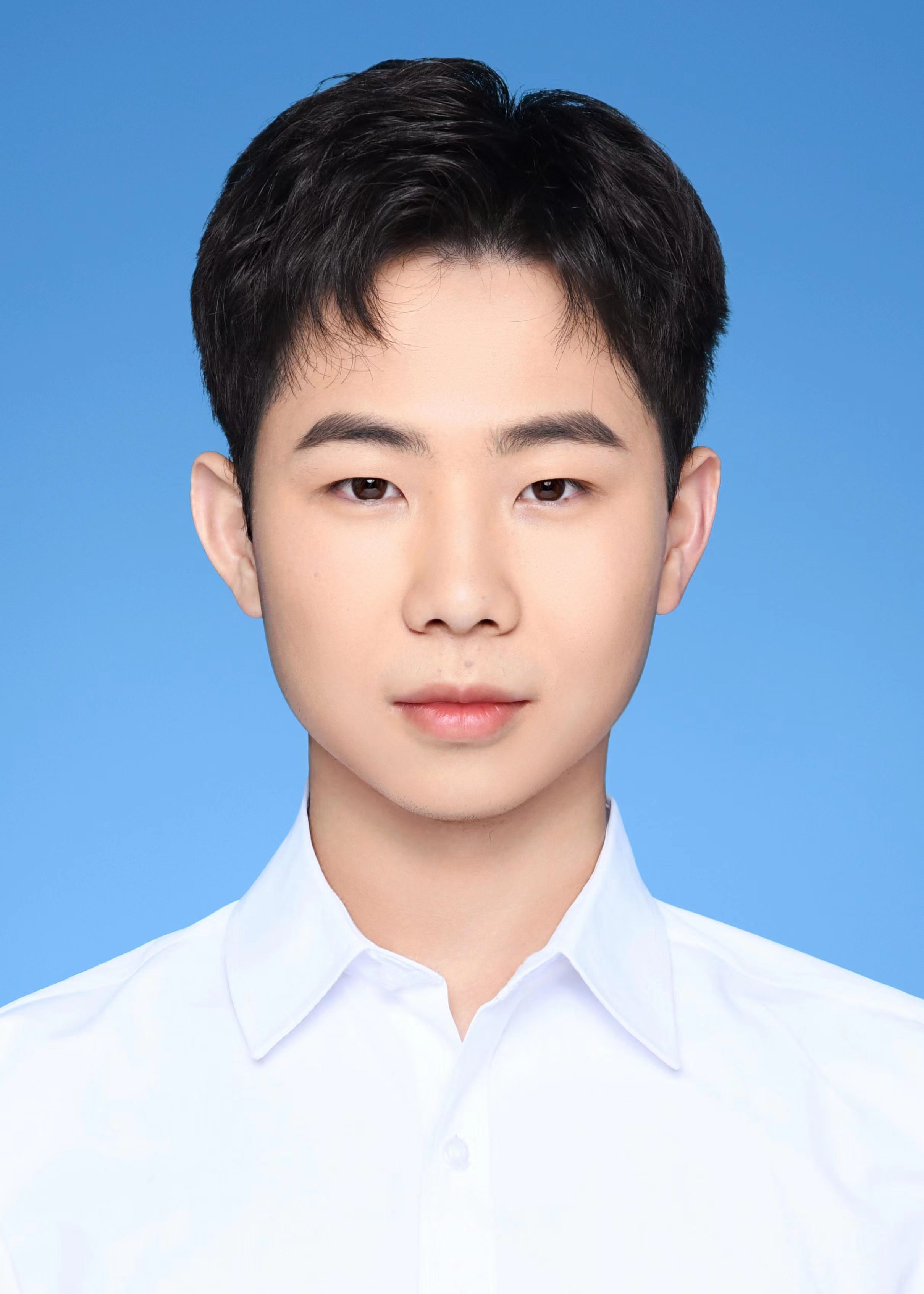}}]{Hanzhu Chen}
  received the B.Sc. degree in Computer Science and Technology from Southwest University, Chongqing, China, in 2021. He is currently a graduate student in the School of Data Science at University of Science and Technology of China, Hefei, China. His research interests include graph representation learning and natural language processing.
\end{IEEEbiography}

\begin{IEEEbiography}[{\includegraphics[width=1in,height=1.5in,clip,keepaspectratio]{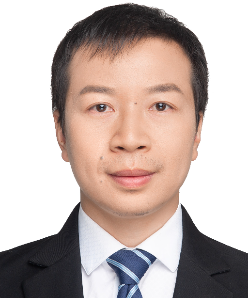}}]{Defu Lian} (Member, IEEE) received the B.E. and Ph.D. degrees in computer science from the University of Science and Technology of China (USTC), Hefei, China, in 2009 and 2014, respectively. He is currently a Professor with the School of Computer Science and Technology, USTC. He has published prolifically in refereed journals and conference proceedings, such as IEEE TRANSACTIONS ON KNOWLEDGE AND DATA ENGINEERING (TKDE), ACM Transactions on Information Systems (TOIS), Knowledge Discovery and Data Mining (KDD), International Joint Conference on Artificial Intelligence (IJCAI), AAAI Conference on Artificial Intelligence (AAAI), Web Search and Data Mining (WSDM), and International World Wide Web Conference (WWW). His general research interests include spatial data mining, recommender systems, and learning to hash. Dr. Lian has served regularly on the program committee of a number of conferences and is a reviewer for the leading academic journals.
\end{IEEEbiography}

\begin{IEEEbiography}[{\includegraphics[width=1in,height=1.5in,clip,keepaspectratio]{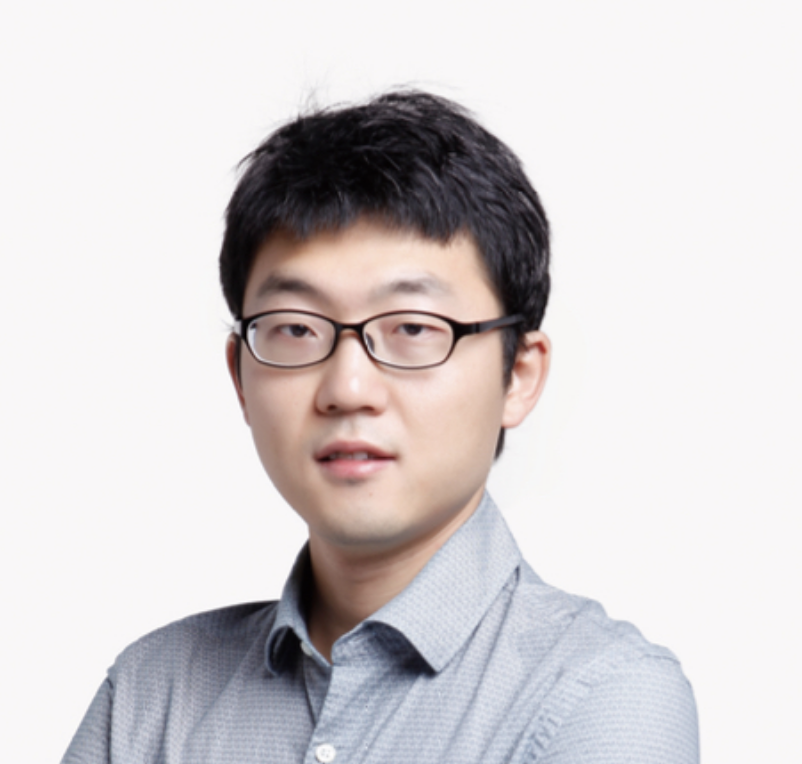}}]{Zheng Wang}  (Menber, IEEE)  received his Ph.D. degree from Tsinghua University in 2011 and worked as a research fellow in Arizona State University in 2011-2014, then as a research faculty in the University of Michigan at Ann Arbor in 2014-2016. He has received several awards, including best research paper award runner-up in KDD and best paper award in IEEE International Conference in Social Computing (SocialCom). He served as the Area Chair and (Senior) PC member of leading conferences, such as ICML, NIPS, KDD and IJCAI, and gave tutorial in KDD, IJCAI and ICDM. Now he is working on AI for drug and structured data analysis as a researcher at Alibaba DAMO Academy.
\end{IEEEbiography}

\begin{IEEEbiography}[{\includegraphics[width=1in,height=1.25in,clip,keepaspectratio]{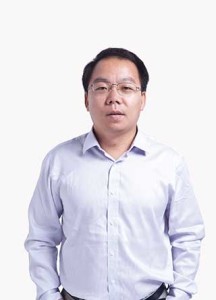}}]{Jieping Ye} (Fellow, IEEE) is currently the VP of the Alibaba Group, Hangzhou, China, where he is the Head of the CityBrain Lab, DAMO Academy.
His research interests include big data, machine learning, and artificial intelligence with applications in transportation, smart city, and biomedicine.
Dr. Ye was elevated to an IEEE Fellow in 2019 and named an ACM Distinguished Scientist in 2020 for his contributions to the methodology and application of machine learning and data mining.
He won the NSF CAREER Award in 2010. His papers have been selected for the Outstanding Student Paper at the International Conference on Machine Learning (ICML) in 2004, the ACM SIGKDD Conference on Knowledge Discovery and Data Mining (KDD) Best Research Paper Runner Up in 2013, and the KDD Best Student Paper Award in 2014. He won the First Place in the 2019 INFORMS Daniel H. Wagner Prize, one of the top awards in operation research practice.
He has served as a Senior Program Committee/Area Chair/Program Committee ViceChair of many conferences, including the Conference and Workshop on Neural Information Processing Systems (NeurIPS), ICML, KDD, International Joint Conference on Artificial Intelligence (IJCAI), the IEEE International Conference on Data Mining (ICDM), and the SIAM International Conference on Data Mining (SDM). He has served as an Associate Editor for Data Mining and Knowledge Discovery, IEEE TRANSACTIONS ON KNOWLEDGE AND DATA ENGINEERING, and IEEE TRANSACTIONS ON PATTERN ANALYSIS AND MACHINE INTELLIGENCE.
\end{IEEEbiography}

\begin{IEEEbiography}[{\includegraphics[width=1in,height=1.25in,clip,keepaspectratio]{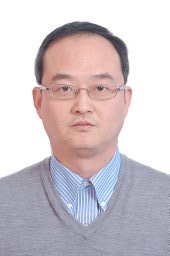}}]{Feng Wu}  (Fellow, IEEE) 
received the B.S. degree in electrical engineering from Xidian University in 1992, and the M.S. and Ph.D. degrees in computer science from the Harbin Institute of Technology in 1996 and 1999, respectively. He is currently a Professor with the University of Science and Technology of China, where he is also the Dean of the School of Information Science and Technology. Before that, he was a Principal Researcher and the Research Manager with Microsoft Research Asia. His research interests include image and video compression, media communication, and media analysis and synthesis. He has authored or coauthored over 200 high quality articles (including several dozens of IEEE Transaction papers) and top conference papers on MOBICOM, SIGIR, CVPR, and ACM MM. He has 77 granted U.S. patents. His 15 techniques have been adopted into international video coding standards. As a coauthor, he received the Best Paper Award at 2009 IEEE Transactions on Circuits and Systems for Video Technology, PCM 2008, and SPIE VCIP 2007. He also received the Best Associate Editor Award from IEEE Circuits and Systems Society in 2012. He also serves as the TPC Chair for MMSP 2011, VCIP 2010, and PCM 2009, and the Special Sessions Chair for ICME 2010 and ISCAS 2013. He serves as an Associate Editor for IEEE Transactions on Circuits and Systems for Video Technology, IEEE Transactions ON Multimedia, and several other international journals.
\end{IEEEbiography}




\clearpage

\appendices

\section{Linear Feature Convolution and Non-linear Transformation in GNNs}\label{sec:decouple_gnn}

In this section, we introduce linear feature convolution and non-linear transformation in GNNs including GCN, REVGAT, GAMLP, SAGN, and GAT. Let the set of $\{1,2,\dots,p\}$ be $[p]$.

Many GNNs \cite{gcn, gcnii, gat, gamlp, deq_gcn} iteratively update node representations in two stages: linear feature convolution and non-linear transformation
\begin{align} \label{eqn:feature_prop}
    \mat{H}^{(i+1/2)} &= \phi^{(i+1)}(\mat{A}; \theta) \mat{H}^{(i)};  \,\text{(linear feature convolution)}\\ \label{eqn:non-linear} 
    \mat{H}^{(i+1)} &= \psi^{(i+1)}(\mat{H}^{(i+1/2)}; \theta),  \,\text{(non-linear transformation)} 
\end{align}
where $\phi^{(i+1)}(\mat{A}; \theta) \in \mathbb{R}^{|\mathcal{V}| \times |\mathcal{V}|}$ is the diffusion matrix constructed from the adjacent matrix $\mat{A}$.
An $N$-layer GNN iteratively takes node features $\mat{H}^{(0)}=\mat{F}^{(\beta)}$ as input and outputs node representations $\mat{H} = \mat{H}^{(N)}$ at the $N$-th layer, i.e.,
\begin{align*}
    \mat{H} = \psi^{(N)} \circ \phi^{(N)} \circ  \psi^{(N-1)} \circ \phi^{(N-1)} \cdots \phi^{(1)} (\mat{F}^{(\beta)}).
\end{align*}

The notorious neighbor explosion issue is due to the linear feature convolutions $\phi^{(i)}$ rather than the non-linear transformations. Thus, we decouple all linear feature convolutions by changing the order of these operations as follows
\begin{align}\nonumber
    \mat{H} & \approx  (\phi^{(N)} \circ \phi^{(N-1)} \cdots \phi^{(1)} ) (\psi^{(N)}  \circ  \psi^{(N-1)} \cdots \psi^{(1)}) (\mat{F}^{(\beta)})\\ 
    & = \phi (\mat{A}; \theta) \psi( \mat{F}^{(\beta)}; \theta ), \label{eqn:decouple}
\end{align}
where $\phi (\mat{A}; \theta) = \prod_{i=1}^n \phi^{(i)}(\mat{A}; \theta) \in \mathbb{R}^{|\mathcal{V}| \times |\mathcal{V}|} $ is still a diffusion matrix and $\psi = \psi^{(N)}  \circ  \psi^{(N-1)} \circ \cdots \psi^{(1)} $ is a non-linear multi-layer perceptron.  

\subsection{Relation between Spectral-based Formulation and Label Deconvolution}

By Equation \eqref{eqn:decouple}, we decouple a given GNN into a diffusion matrix $\phi$ and a non-linear multi-layer perceptron $\psi$. We discuss the effects of $\phi$ and $\psi$ for label deconvolution respectively.

\udfsection{Effect of $\phi$.} In Section 3.2, we introduce label deconvolution to parameterize the inverse of the diffusion matrix $\phi$. In the derivation, we consider the general spectral-based GNNs with $\phi (\mat{A}; \theta) = \sum_{i=1}^N \theta_i^{(\phi)} \hat{\mat{A}}^i$ whose expressive power is stronger than a specific GNN (e.g. $\phi(\hat{\mathbf{A}}; \theta) = \hat{\mathbf{A}}^N$ for GCN).
The feature convolution layers of GNNs help select the number of the label deconvolution layers $N$. As shown in Fig. 5, LD achieves the best performance when $N$ is equal to the number of feature convolution layers of GNNs on the ogbn-product and ogbn-protein datasets.
Intuitively, the $N$ feature convolution layers provide the useful inductive bias that the prediction (representation) of a node not only depends on its feature but also its $N$-hop neighbors' features.
the $N$ label deconvolution layers provide the same inductive bias during the training of NEs.



\udfsection{Effect of $\psi$.} During the training phase of NEs, the non-linear multi-layer perceptron $\psi$ is behind the node encoder, as shown in Equation (12).
However, the effect of non-linear multi-layer perceptron $\psi$ for the node feature extraction is marginal, as the used node encoder (e.g. ESM2 \cite{esm2} and Bert \cite{bert, deberta}) contain a large number of non-linear layers.

\subsection{Example}\label{sec:example_spe}

\udfsection{GCN \cite{gcn}.} GCN iteratively updates node representations by
\begin{align*}
    \mat{H}^{(i+1)}  = \sigma( \hat{\mat{A}} \mat{H}^{(i)}   \mat{W}^{(i+1)}),
\end{align*}
where $\mat{W}^{(i+1)}$ is the trainable weights of the $(i+1)$-th layer and $\sigma$ is an activation function (e.g. ReLU, TanH, and Sigmoid). 
$\hat{\mat{A}} = \mat{D}^{-1}\mat{A}$ is the normalized adjacent matrix, where $\mat{D}$ is the degree matrix with $\mat{D}_{i,i}= \sum_{j} \mat{A}_{i,j}$ and $\mat{D}_{i,j}=0, \forall i \neq j$.
By letting $\mat{H}^{(i+1/2)} = \hat{\mat{A}} \mat{H}^{(i)}$, the corresponding linear feature convolution and non-linear transformation are
\begin{align*}
    \mat{H}^{(i+1/2)} &= \phi^{(i+1)}(\mat{A}; \theta) \mat{H}^{(i)} = \hat{\mat{A}} \mat{H}^{(i)},\\
    \mat{H}^{(i+1)} &= \psi^{(i+1)}(\mat{H}^{(i+1/2)}; \mat{W}^{(l+1)}) = \sigma( \mat{H}^{(i+1/2)} \mat{W}^{(i+1)})
\end{align*}
respectively. Thus, the resulting diffusion matrix and the multi-layer perceptron are
\begin{align*}
    \phi(\mat{A}; \theta) &= \hat{\mat{A}}^N,\\
    \psi(\mat{X}; \mat{W}^{(i)},i=1,2,\dots,N) &= \sigma( \cdots \sigma( \mat{X} \mat{W}^{(1)}) \cdots \mat{W}^{(N)} ).
\end{align*}

\udfsection{REVGAT \cite{deq_gcn}.} At the $(i+1)$-th layer, REVGAT first partition $\mat{H}^{(i)}$ into $C$ groups $(\mat{H}^{(i)}_1, \mat{H}^{(i)}_2, \dots ,\mat{H}^{(i)}_C)$ and then update node representations in each group by
\begin{align*}
    \overline{\mat{H}}^{(i+1)}_0 &= \sum_{j=2}^C \mat{H}^{(i)}_j;\\
    \overline{\mat{H}}^{(i+1)}_j  &= \GAT^{(i+1)}_{j-1} + \mat{H}^{(i)}_j,\,j \in [C],\\
    \GAT^{(i+1)}_{j-1} &= \sigma( \mat{A}(\overline{\mat{H}}^{(i+1)}_{j-1}, \theta) \overline{\mat{H}}^{(i+1)}_{j-1}   \mat{W}^{(i+1)}_j );\\
    \mat{H}^{(i+1)} &= \MLP(\overline{\mat{H}}^{(i+1)});
\end{align*}
where $\mat{A}(\mat{H}^{(i)}, \theta)$ is a learnable attention matrix, $\mat{W}^{(i+1)}$ is the trainable weights of the $(i+1)$-th layer, and $\sigma$ is an activation function. 
As recent works show that we can replace the graph attention with a normalized random vector to achieve similar performance \cite{imp}, we approximate $\mat{A}(\mat{H}^{(i)}, \theta)$ by a fixed normalized adjacent matrix $\hat{\mat{A}}$.
There exists $NC$ feature convolutions $\phi^{(i+1)}_j$ and non-linear transformation $\psi^{(i+1)}_j$.
If $j=1$, by letting $\mat{H}^{(i+1)}_{1/2} = \hat{\mat{A}} \mat{H}^{(i+1)}$, the corresponding linear feature convolution is
\begin{align*}
    \mat{H}^{(i+1)}_{1/2} &= \phi^{(i+1)}_1(\mat{A}; \theta) \overline{\mat{H}}^{(i+1)} = \hat{\mat{A}} \mat{H}^{(i+1)};
\end{align*}
and the non-linear transformation $\psi^{(i+1)}_1$ is
\begin{align*}
    \overline{\mat{H}}^{(i+1)}_1  &= \GAT^{(i+1)}_{0} + \mat{H}^{(i)}_1,\\
    \GAT^{(i+1)}_{0} &= \sigma(\sum_{j=2}^C [\mat{H}^{(i+1)}_{1/2}]_j   \mat{W}^{(i+1)}_1 ).
\end{align*}
If $j \in [2,C-1]$, by letting $\mat{H}^{(i+1)}_{j-1/2} = \hat{\mat{A}} \overline{\mat{H}}^{(i+1)}_{j-1}$, the corresponding linear feature convolution is
\begin{align*}
    \mat{H}^{(i+1)}_{j-1/2} &= \phi^{(i+1)}_j(\mat{A}; \theta) \overline{\mat{H}}^{(i+1)}_{j-1} = \hat{\mat{A}} \overline{\mat{H}}^{(i+1)}_{j-1},
\end{align*}
and the non-linear transformation $\psi^{(i+1)}_j$ compute $\overline{\mat{H}}^{(i+1)}_j$ by
\begin{align*}
    \overline{\mat{H}}^{(i+1)}_j  &= \GAT^{(i+1)}_{j-1} + \mat{H}^{(i)}_j,\\
    \GAT^{(i+1)}_{j-1} &= \sigma( \mat{H}^{(i+1)}_{j-1/2}   \mat{W}^{(i+1)}_j ).
\end{align*}
If $j = C$, by letting $\mat{H}^{(i+1)}_{C-1/2} = \hat{\mat{A}} \overline{\mat{H}}^{(i+1)}_{C-1}$, the corresponding linear feature convolution is
\begin{align*}
    \mat{H}^{(i+1)}_{C-1/2} &= \phi^{(i+1)}_C(\mat{A}; \theta) \overline{\mat{H}}^{(i+1)}_{C-1} = \hat{\mat{A}} \overline{\mat{H}}^{(i+1)}_{C-1},
\end{align*}
and the non-linear transformation $\psi^{(i+1)}_C$ compute $\overline{\mat{H}}^{(i+1)}$ by
\begin{align*}
    \overline{\mat{H}}^{(i+1)}_C  &= \GAT^{(i+1)}_{C-1} + \mat{H}^{(i)}_C,\\
    \GAT^{(i+1)}_{C-1} &= \sigma( \mat{H}^{(i+1)}_{C-1/2}   \mat{W}^{(i+1)}_C )\\
    \mat{H}^{(i+1)} &= \MLP(\overline{\mat{H}}^{(i+1)}).
\end{align*}

Thus, the resulting diffusion matrix is
\begin{align*}
    \phi(\mat{A}; \theta) &= \hat{\mat{A}}^{NC},
\end{align*}
and the resulting $\psi$ is an $N$-layer reversible network
\begin{align*}
    \overline{\mat{H}}^{(i+1)}_0 &= \sum_{j=2}^C \mat{H}^{(i)}_j;\\
    \overline{\mat{H}}^{(i+1)}_j  &= \sigma( \overline{\mat{H}}^{(i+1)}_{j-1}   \mat{W}^{(i+1)}_j ) + \mat{H}^{(i)}_j,\,j \in [C],\\
    \mat{H}^{(i+1)} &= \MLP(\overline{\mat{H}}^{(i+1)}).
\end{align*}

\udfsection{GAMLP \cite{gamlp}.} Given the node features $\mat{X}$ and the label embeddings $\mat{H}_Y$, GAMLP learns node representations by
\begin{align*}
    \mat{H}  &= \MLP(\mat{H}_X) + \beta \MLP(\mat{H}_Y);\\
    \mat{H}_X &= \sum_{i=0}^N \theta_i \hat{\mat{A}}^i \mat{X},
\end{align*}
where $\MLP$ is a multi-layer perception, $\beta$ is a hyper-parameter, and $\theta_i$ is the learned weights (e.g., recursive attention and JK attention).
As the computation costs of the label embeddings $\mat{H}_Y$ are very expensive during the training phase of NEs, we set $\beta=0$. We use $\theta_i \in \mathbb{R}$ in the derivation of this paper, as the scalar weight is powerful enough to produce arbitrary node predictions under some mild conditions\cite{linear_gnn}.
Thus, GAMLP performs the linear feature convolution and the non-linear transformation only once by
\begin{align*}
    \mat{H}^{(1/2)} &= \phi^{(1)}(\mat{A}; \theta) \mat{H}^{(0)} = (\sum_{i=0}^N \theta_i \hat{\mat{A}}^i ) \mat{H}^{(0)},
\end{align*}
and
\begin{align*}
    \mat{H}^{(1)} &= \psi^{(1)}(\mat{H}^{(1/2)},) = \MLP(\mat{H}^{(1/2)}).
\end{align*}
Thus, $\phi$ and $\psi$ are
\begin{align*}
    \phi(\mat{A}; \theta) &= \sum_{i=0}^N \theta_i  \hat{\mat{A}}^{(i)},\\
    \psi(\mat{X}) &= \MLP(X).
\end{align*}

\udfsection{SAGN \cite{sagn}.} Given the node features $\mat{X}$, SAGN learns node representations by
\begin{align*}
    \mat{Z}  &= \mat{X} || \hat{\mat{A}} \mat{X} || \hat{\mat{A}}^2 \mat{X} || \dots \hat{\mat{A}}^N \mat{X} ;\\
    \mat{H} &= \MLP(\mat{Z}) = \sum_{i=0}^N \MLP_i(\hat{\mat{A}}^i \mat{X}),
\end{align*}
where $||$ denotes concatenation.
As the nonlinearity of $\MLP_i$ is unnecessary for SAGN to reach high expressiveness \cite{linear_gnn}, we remove the nonlinearity of $\MLP_i$ by
\begin{align*}
    \mat{H} &= \sum_{i=0}^N \MLP_i(\hat{\mat{A}}^i \mat{X})\approx \sum_{i=0}^N \hat{\mat{A}}^i \mat{X} \theta_i.
\end{align*}
We follow the derivation of GAMLP to approximate $\phi$ and $\psi$ by
\begin{align*}
    \phi(\mat{A}; \theta) &= \sum_{i=0}^N \theta_i  \hat{\mat{A}}^{(i)},\\
    \psi(\mat{X}) &= \MLP_0(X).
\end{align*}
We preserve the architecture $\MLP_0$ of the true SAGN to alleviate the learning bias.

\udfsection{GAT \cite{gat}.} GAT iteratively updates node representations by
\begin{align*}
    \mat{H}^{(i+1)}  = \sigma( \MLP( \mat{A}(\mat{H}^{(i)}, \mathcal{E}, \theta) \mat{H}^{(i)}   \mat{W}^{(i+1)} )),
\end{align*}
where $\mat{A}(\mat{H}^{(i)}, \mathcal{E}, \theta)$ is a learnable attention matrix, $\mathcal{E}$ denotes the edge features, $\MLP$ is a multi-layer perception introduced by NGNN \cite{ngnn}, $\mat{W}^{(i+1)}$ is the trainable weights of the $(i+1)$-th layer, and $\sigma$ is an activation function (e.g. ReLU, TanH, and Sigmoid). 
As recent works show that we can replace the graph attention with a normalized random vector to achieve similar performance \cite{imp}, we approximate $\mat{A}(\mat{H}^{(i)}, \mathcal{E}, \theta)$ by a fixed normalized adjacent matrix $\hat{\mat{A}}$.
By letting $\mat{H}^{(i+1/2)} = \hat{\mat{A}} \mat{H}^{(i)}$, the corresponding linear feature convolution and non-linear transformation are
\begin{align*}
    \mat{H}^{(i+1/2)} &= \phi^{(i+1)}(\mat{A}; \theta) \mat{H}^{(i)} = \hat{\mat{A}} \mat{H}^{(i)},\\
    \mat{H}^{(i+1)} &= \psi^{(i+1)}(\mat{H}^{(i+1/2)}; \mat{W}^{(i+1)}) \\
    &= \sigma( \MLP(\mat{H}^{(i+1/2)} \mat{W}^{(i+1)}))
\end{align*}
respectively. Thus, the resulting diffusion matrix and the multi-layer perceptron are
\begin{align*}
    \phi(\mat{A}; \theta) &= \hat{\mat{A}}^N,\\
    \psi(\mat{X}) &= \sigma(\MLP( \cdots \sigma( \MLP(\mat{X} \mat{W}^{(1)})) \cdots \mat{W}^{(N)} )).
\end{align*}

\subsection{Expressiveness of Spectral-based GNNs} \label{sec:expressiveness_gnn}

In this section, we show that the approximation \eqref{eqn:decouple} does not compromise the expressiveness of GNNs. Specifically, following \cite{linear_gnn}, we show that the spectral-based GNN \eqref{eqn:decouple} can differentiate all non-isomorphic nodes if the Laplace matrix $\hat{\mathbf{L}} = \mathbf{I}  - \hat{\mathbf{A}}$ has no multiple eigenvalues and the node features $\mathbf{F}^{\beta}$ contain all frequency components of $\hat{\mathbf{L}}$.
We also analyze why the no-multiple-eigenvalue and no-missing-frequency conditions are largely satisfied in practice.

\begin{assumption}\label{ass:eigen}
    Let $\hat{\mathbf{L}} = \mathbf{U}\Lambda \mathbf{U}^{\top}$ denote the eigendecomposition of the Laplace matrix $\hat{\mathbf{L}} = \mathbf{I}  - \hat{\mathbf{A}}$, where $\Lambda=\text{Diag}(1-\lambda_1, 1-\lambda_2, \dots, 1-\lambda_n)$ is the diagonal matrix constituted by the eigenvalues $\lambda_i$ and $\mathbf{U}=(\mathbf{u}_1, \mathbf{u}_2, \dots, \mathbf{u}_n)$ is constituted by the eigenvectors $\mathbf{u}_i$.
    We assume that the no-multiple-eigenvalue and no-missing-frequency conditions hold, i.e., $\lambda_i \neq \lambda_j$ for $i \neq j$ and $u_{i}^{\top} \mathbf{F}^{\beta} \neq \mathbf{0}$ for $i=1,2,\dots,n$.
\end{assumption}

As the eigendecomposition is very expensive for large-scale datasets in our experiments, we quote Table 7 in \cite{linear_gnn} on ten real-world datasets in Table \ref{tab::cite_datasets}.
\begin{table*}[t]
\centering
\caption{Dataset statistics \cite{linear_gnn}. $N_{\text{miss}}$ is the number of missing frequency components. $R_{\text{multi}}$ is the ratio of multiple eigenvalues in all different Laplacian eigenvalues (\%).}\label{tab::cite_datasets}
\vskip 0.15in
\begin{center}
\begin{small}
\setlength{\tabcolsep}{1mm}
{\begin{tabular}{ccccccccccc}
\hline
Datasets & Cora & CiteSeer & PubMed & Computers & Photo  & Chameleon & Squirrel & Actor & Texas & Cornell \\
\hline
$|V|$ & 2708 & 3327 & 19717 & 13752 & 7650 & 2277 & 5201 & 7600 & 183 & 183     \\
$|E|$ & 5278 & 4552 & 44324 & 245861 & 119081 & 31371 & 198353& 26659 & 279 & 277     \\
$N_{\text{miss}}$ & 0 &0 &0& 0 &0 &0& 0 &0 &0& 0\\
$R_{\text{multi}}$ & 1.49&3.25&0.03&0.16&0.22&0.22&0.00&0.13&0.00&0.00\\
\hline
\end{tabular}}
\end{small}
\end{center}
\vskip -0.1in
\end{table*}
In the ten real-world benchmark datasets, on average less than 1\% of eigenvalues are multiple and no frequency component is missing.
In Table \ref{tab::cite_datasets}, the datasets Cora, CiteSeer, and PubMed are citation networks like ogbn-arxiv in our experiments; the datasets Computers and Photo are co-purchase networks like ogbn-products in our experiments.
Therefore, Assumption \ref{ass:eigen} can be largely satisfied in our experiments.

Under Assumption  \ref{ass:eigen}, we give the universality theorem of spectral GNNs \cite{linear_gnn}. 
The theorem points out that our approximation \eqref{eqn:decouple} does not compromise the expressiveness of GNN, i.e., its ability to distinguish non-isomorphic nodes/graphs.
\begin{theorem}
    If the Laplace matrix $\hat{\mathbf{L}} = \mathbf{I}  - \hat{\mathbf{A}}$ has no multiple eigenvalues and the node features $\mathbf{F}^{\beta}$ contain all frequency components of $\hat{\mathbf{L}}$, then the spectral-based GNN \eqref{eqn:decouple} can produce any one-dimensional prediction for all non-isomorphic nodes.
\end{theorem}
\begin{proof}
    The spectral-based GNN is $\mathbf{H} = \phi (\mat{A}; \theta) \psi( \mat{F}^{(\beta)}; \theta ) = \sum_{i=0}^{N} \theta_i^{(\phi)} \hat{\mat{A}}^i \mat{F}^{(\beta)} \mathbf{W}$, where we replace the multi-layer perceptron $\psi$ with a linear layer $\psi( \mat{F}^{(\beta)}) = \mat{F}^{(\beta)} \mathbf{W}$.
    Due to the universal approximation property of the multi-layer perceptron, the conclusion for a linear layer also holds for the multi-layer perception.

    Let $g(\lambda_i) = \sum_{i=0}^{n-1} \theta_i^{(\phi)} \lambda_i$. Then the node embeddings become $\mathbf{H} = \mathbf{U}\text{diag}(g(\lambda_1), g(\lambda_2), \dots, g(\lambda_n))\mathbf{U}^{\top}\mat{F}^{(\beta)} \mathbf{W}  $. Given the label $\mathbf{Y} \in \mathbb{R}^n$, suppose that $\widetilde{\mathbf{Y}} = \mathbf{U}^{\top} \mathbf{Y}$ and $\widetilde{\mathbf{F}} = \mathbf{U}^{\top} \mat{F}^{(\beta)}$.

    We first prove that there exists $\mathbf{W}^* \in \mathbb{R}^d$ such that all elements of $\widetilde{\mathbf{F}} \mathbf{W}^* \in \mathbb{R}^n$ are not zero. Due to the no-missing-frequency condition, we have $\widetilde{\mathbf{F}}_{i,:} = u_{i}^{\top} \mathbf{F}^{\beta} \neq \mathbf{0}$. Thus, the solution space $V_i \in \mathbb{R}^d = \{\mathbf{W}: \widetilde{\mathbf{F}}_{i,:} \mathbf{W} = \mathbf{0}\}$ is a proper subspace of $\mathbb{R}^d$ (the dimensionality of $V_i$ is $(d-1)$). Therefore, we have $\cup_{i=1}^n V_i \neq \mathbb{R}^d$. All parameters $\mathbf{W}^* \in \mathbb{R}^{d} - \cup_{i=1}^n V_i$ can meet the requirements.

    Then, if there exists a polynomial such that $g^*(\lambda_i) = \widetilde{\mathbf{Y}}_i / (\widetilde{\mathbf{F}} \mathbf{W}^*)_i = \mathbf{R}_i $ for $i=1,2,\dots,n$, then the spectral-based GNN can produce the prediction $\mathbf{Y}$.
    Thus, the parameters $\theta_i^{(\phi)}$ are the solution to Equation $\mathbf{B} \Theta = \mathbf{R}$ with $\mathbf{B}_{ij} = \lambda_i^{j-1}$. As  $\lambda_i$ are different from each other, the Vandermonde matrix $\mathbf{B}$ is nonsingular and the solution to Equation $\mathbf{B} \Theta = \mathbf{R}$ always exists.
    Therefore, the spectral-based GNN can produce any one-dimensional prediction for all non-isomorphic nodes.
\end{proof}

\subsection{Emperical Results for Approximation in Equation \eqref{eqn:decouple}}\label{sec:exp_approx}

Besides the theoretical analysis in Appendix \ref{sec:expressiveness_gnn}, we also conduct experiments to demonstrate that our approximation does not compromise the expressiveness of GNNs.
Specifically, we use the embeddings inferred by pre-trained NEs $\mathbf{X}_{\text{PNE}}$ as the node features and then train the spectral-based GNNs $\mathbf{H} = \sum_{i=0}^{N} \theta_i^{(\phi)} \hat{\mat{A}}^i \text{MLP}(\mat{F}^{(\beta)})$ by the full-batch gradient descent.
For GCN, the corresponding spectral-based approximation is $\mathbf{H} = \hat{\mat{A}}^N \text{MLP}(\mat{F}^{(\beta)})$  (denoted by SpeGCN). For RevGAT, the corresponding spectral-based approximation is $\mathbf{H} = \sum_{i=0}^{NC} \theta_i^{(\phi)} \hat{\mat{A}}^i \text{MLP}(\mat{F}^{(\beta)})$ with parameters $\theta_i^{(\phi)}$  (denoted by SpeGNN).

Fig. \ref{fig:curves} shows the training curves of GCN, RevGAT, and their spectral-based approximations. The training behavior of GCN is very similar to that of SpeGCN, showing that swapping linear convolutional layers and non-linear layers of GNNs does not significantly affect model performance.
Although RevGAT introduces the attention mechanism as a learnable convolutional operator, the prediction performance of RevGAT is very close to SpeGNN.
Thus, the spectral-based approximation can maintain the expressive power and prediction performance of the original GNN.

\begin{figure}[t]
	\centering

    \begin{subfigure}{0.48\linewidth}
		\centering
		\includegraphics[width=1.0\linewidth]{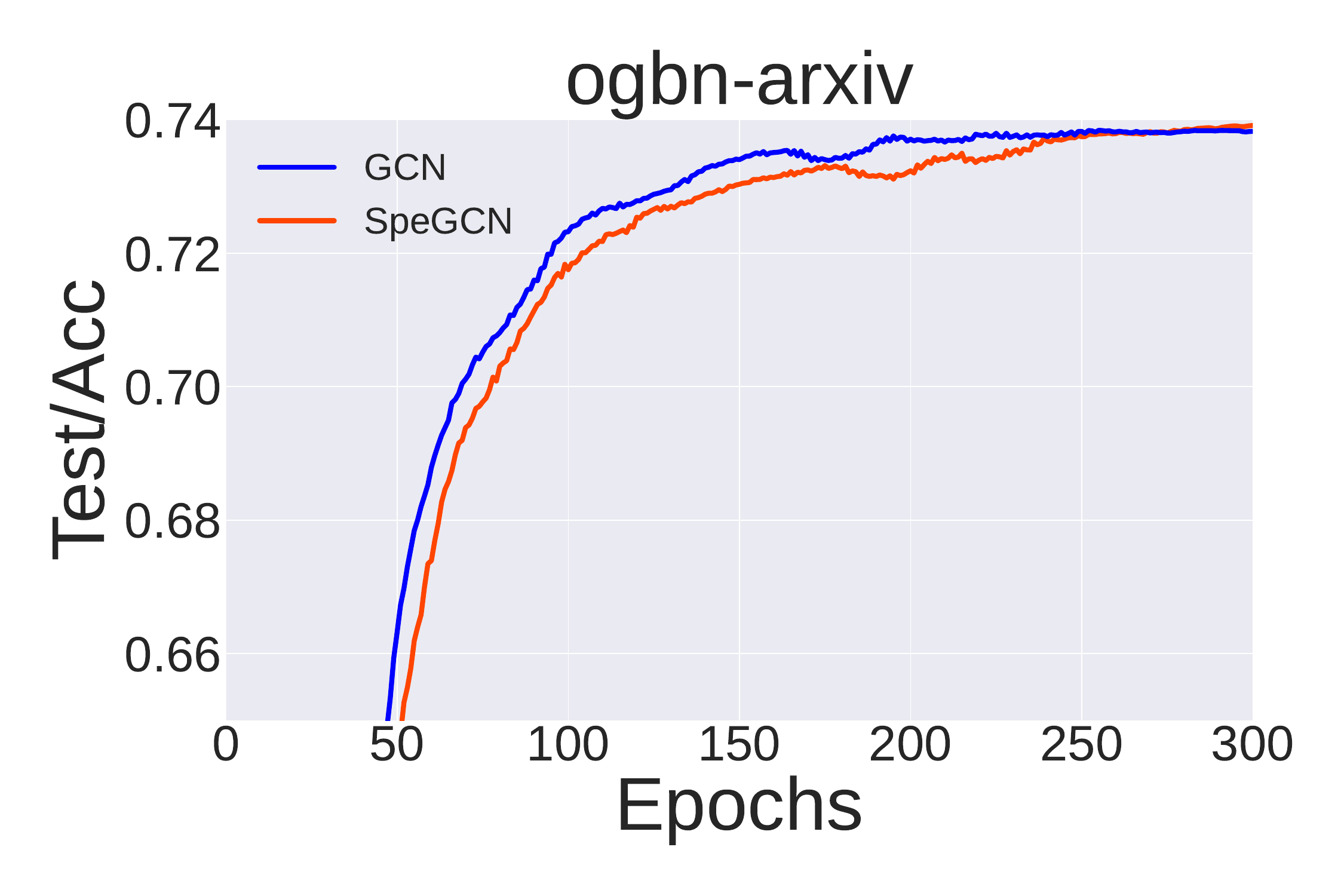}
		\caption{GCN and SpeGCN}
		\label{fig:curves_gcn}
	\end{subfigure}
    \begin{subfigure}{0.48\linewidth}
		\centering
		\includegraphics[width=1.0\linewidth]{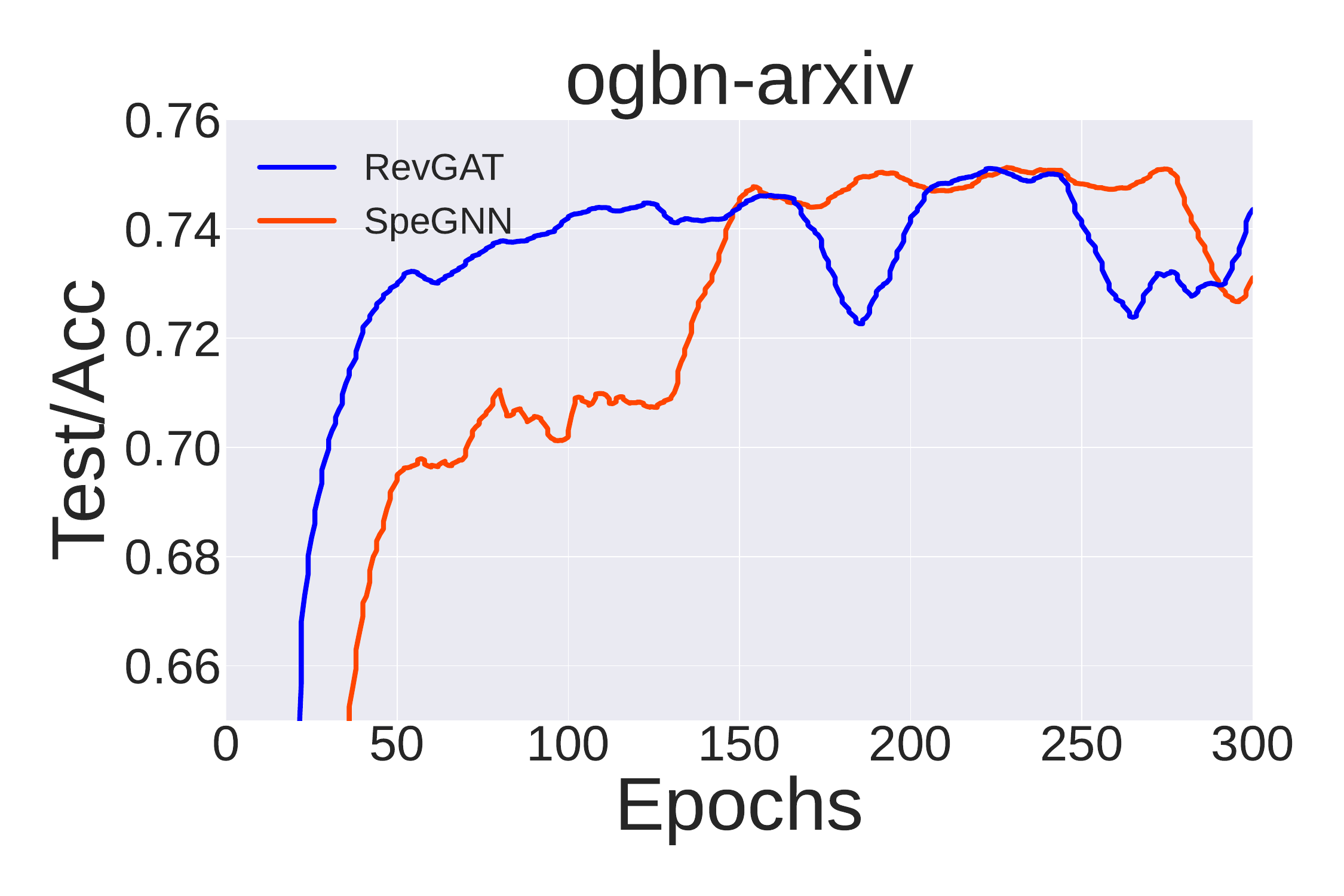}
		\caption{RevGAT and SpeGNN}
		\label{fig:curves_revgat}
	\end{subfigure}
        \caption{The training curves of GNNs and their spectral-based approximation.}

    \label{fig:curves}
\end{figure}

\section{Detailed proofs}\label{sec:proof}

\subsection{Proof of Proposition 1}

It follows from Lemma 1 that $\phi(\mat{A}; \theta)^{-1}$ is expressed as a linear combination of the matrix powers of $\hat{\mat{A}}$, i.e.,
\begin{align*}
    \phi(\mat{A}; \theta)^{-1} &= \sum_{j=0}^{|\mathcal{V}|-1} p_j \phi(\mat{A}; \theta)^j\\
    &= \sum_{j=0}^{|\mathcal{V}|-1} p_j  (\sum_{i=0}^{N} \theta_i^{(\phi)} \hat{\mat{A}}^i)^j\\
    &= \sum_{i=0}^{N(|\mathcal{V}|-1)} q_i \hat{\mat{A}}^i.
\end{align*}
It follows from Lemma 2 that there exists $\gamma_i$ such that
\begin{align*}
    \phi(\mat{A}; \theta)^{-1} 
    = \sum_{i=0}^{|\mathcal{V}|-1} \gamma_i \hat{\mat{A}}^i.
\end{align*}

\subsection{Solution to the motivating example in Section 4.1}
\label{sec:solution_example}

The node attributes, the node labels, and the normalized adjacent matrix are $\mat{X} = (\mat{e}_1, \mat{e}_2, \mat{e}_2, \mat{e}_3)^{\top}=\begin{bmatrix} 1 & 0 & 0 \\ 0 & 1 & 0 \\ 0 & 1 & 0\\ 0 & 0 & 1 \end{bmatrix}$, $\mat{Y} = (\mat{e}_2, \mat{e}_1, \mat{e}_3, \mat{e}_2)^{\top}=\begin{bmatrix} 0 & 1 & 0 \\ 1 & 0 & 0 \\ 0 & 0 & 1\\ 0 & 1 & 0 \end{bmatrix}$, and
$
\hat{\mat{A}} = \begin{bmatrix} 0 & 1 & 0 & 0 \\ 1 & 0 & 0 & 0 \\ 0 & 0 & 0 & 1 \\ 0 & 0 & 1 & 0  \end{bmatrix}.
$

GLEM first solve
$$
    \beta_{GLEM} = \arg\min_{\beta} \|\mat{X}\beta - \mat{Y}\|_F^2 = \begin{bmatrix} 0 & 1 & 0 \\ 0.5 & 0 & 0.5 \\ 0 & 1 & 0 \end{bmatrix}.
$$
Then, the final prediction of GLEM is
$$
    \mat{A}\mat{F}_{GLEM} = \mat{A}\mat{X}\beta_{GLEM} = \begin{bmatrix} 0.5 & 0 & 0.5 \\ 0 & 1 & 0 \\ 0 & 1 & 0\\ 0.5 & 0 & 0.5 \end{bmatrix}=\mat{Y}.
$$

LD first solve
\begin{align*}
    \beta_{LD},\gamma_0^*, \gamma_1^* &= \arg\min_{\beta, \gamma_0, \gamma_1} \|\gamma_0 \mat{X}\beta + \gamma_1 \mat{A} \mat{X}\beta - \mat{Y}\|_F^2 \\
    &= \begin{bmatrix} 1 & 0 & 0 \\ 0 & 1 & 0 \\ 0 & 0 & 1 \end{bmatrix}, 0, 1.
\end{align*}
Then, the final prediction of LD is
$$
    \mat{A}\mat{F}_{LD} = \mat{A}\mat{X}\beta_{LD} = \begin{bmatrix} 0 & 1 & 0 \\ 1 & 0 & 0 \\ 0 & 0 & 1\\ 0 & 1 & 0 \end{bmatrix}.
$$

\subsection{Proof of Theorem 1}

\begin{proof}
    We first show that $\min_{\beta,\theta,\gamma} \loss(\psi(\mat{F}^{(\beta)}; \theta),  \mat{Y}^{(\gamma)}) = 0$.
    Let $\phi(\hat{\mat{A}},\theta^*) = \phi(\mat{A}) $,  $\psi(\mat{F}^*,\theta^*) = \psi(\mat{F}^*) $, and $\mat{F}^{\beta^*} = \mat{F}^*$, where $\phi(\mat{A})$, $\psi(\mat{F}^*)$, and $\mat{F}^*$ are defined in Assumption 1.
    Then, we have $\phi(\hat{\mat{A}},\theta^*) = \phi(\mat{A}) $ is invertible, we have
    \begin{align*}
        \loss( \psi(\mat{F}^{(\beta^*)};\theta^*), \phi(\hat{\mat{A}};\theta^*)^{-1} \mat{Y}) = 0.
    \end{align*}
    It follows from Lemmas 1 and 2 that $\phi(\mat{A}; \theta)^{-1}$ is expressed as a linear combination of the matrix powers of $\hat{\mat{A}}$, i.e., $\phi(\mat{A}; \theta)^{-1} = \sum_{i=0}^{|\mathcal{V}|-1} \gamma_i^* \hat{\mat{A}}^i$.
    Therefore, we have $\min_{\beta,\theta,\gamma} \loss(\psi(\mat{F}^{(\beta)}; \theta),  \mat{Y}^{(\gamma)}) = \loss(\psi(\mat{F}^{(\beta^*)}; \theta^*),  \hat{\mat{Y}}^{(\gamma^*)}) = 0$.

    Let $\beta_{LD} = \beta^*$. Then, we have
    \begin{align*}
        &\min_{ \theta} \, \loss(\GNN(\mat{F}^{(\beta_{LD})}, \mat{A};\theta), \mat{Y}) \\
        = &\loss(\phi(\hat{\mat{A}};\theta^*) \psi(\mat{F}^{(\beta_{LD})};\theta^*), \mat{Y}) = 0.
    \end{align*}
\end{proof}


\section{Independence of Deconvolved Pseudo Labels}

We show that the deconvolved pseudo labels $\mathbf{Y}^{(\gamma)}$ are independent under some mild assumptions.
\begin{theorem}
    Suppose $G$ is a graph with adjacency matrix $A$, whose $n$ nodes have features $\{F_i\}_{i=1}^n$ and corresponding labels $\{Y_i\}_{i=1}^n$.  Further, we assume that all node features are mutually independent, i.e.
$$
\Pr(F_i, F_j) \;=\; \Pr(F_i)\,\Pr(F_j)
\quad\text{for all }i\neq j.
$$
If Assumption \ref{ass:attributes_labels} holds, then the deconvolved pseudo-labels $\hat Y_i$ are conditionally independent given the features.  Specifically, for any two distinct nodes $i$ and $j$,
$$
\Pr\bigl(\hat Y_i,\hat Y_j \mid F_i, F_j\bigr)
\;=\;
\Pr\bigl(\hat Y_i\mid F_i, F_j\bigr)\,\Pr\bigl(\hat Y_j\mid F_i, F_j\bigr).
$$
\end{theorem}
\begin{proof}
    According to Theorem \ref{thm:cayley}, the deconvolved pseudo-labels $\hat{\mathbf{Y}}$ satisfy 
    \begin{align*}
        &\hat{\mathbf{Y}} = \phi(\hat{\mat{A}})^{-1} \mat{Y} =  \psi(\mat{F})\\
        \Rightarrow &\hat{Y}_i = \psi(F_i).
    \end{align*}
    Thus, for nodes $i,j$, we have
    \begin{align*}
        \text{Pr}((F_i,Y_i),(F_j,Y_j))&=\text{Pr}((F_i,\psi(F_i)),(F_j,\psi(F_j)))\\
        &=\text{Pr}((F_i,Y_i),(F_j,Y_j))\\
        &=\text{Pr}((F_i,\psi(F_i)),\psi(F_j)|F_j)\text{Pr}(F_j)\\
        &=\text{Pr}((F_i,\psi(F_i)))\text{Pr}(F_j)\\
        &=\text{Pr}((F_i,\psi(F_i)))\text{Pr}(F_j,\psi(F_j))\\
        &=\text{Pr}((F_i,Y_i))\text{Pr}(F_j,Y_j).
    \end{align*}
    The conclusion follows immediately
    \begin{align*}
        \text{Pr}(\hat{Y}_i,\hat{Y}_j|F_i,F_j)\text{Pr}(F_i,F_j)
& = \text{Pr}((F_i,\hat{Y}_i),(F_j,\hat{Y}_j)) \\
& = \text{Pr}((F_i,\hat{Y}_i))\text{Pr}(F_j,\hat{Y}_j) \\
& = \text{Pr}(\hat{Y}_i|F_i)\text{Pr}(\hat{Y}_j|F_j)\text{Pr}(F_i)\text{Pr}(F_j)\\
& = \text{Pr}(\hat{Y}_i|F_i)\text{Pr}(\hat{Y}_j|F_j)\text{Pr}(F_i, F_j)\\
\Rightarrow \text{Pr}(\hat{Y}_i,\hat{Y}_j|F_i,F_j)& =  \text{Pr}(\hat{Y}_i|F_i)\text{Pr}(\hat{Y}_j|F_j)\\
& =\text{Pr}(\hat{Y}_i|F_i,F_j)\text{Pr}(\hat{Y}_j|F_i,F_j),
    \end{align*}
The last equation of independence is due to $\hat{Y}_i = \psi(F_i)$.
\end{proof}

\end{document}